%% file: acl_latex.tex
\documentclass[11pt]{article}

\usepackage[final]{main}

\providecommand{\nolinenumbers}{}
\providecommand{\linenumbers}{}

\usepackage{times}
\usepackage{latexsym}
\usepackage[T1]{fontenc}
\usepackage[utf8]{inputenc}
\usepackage{microtype}
\usepackage{inconsolata}
\usepackage{graphicx}
\usepackage{fancyhdr}

\input{preamble/_preamble_includes}

\usepackage{newtxmath}

\definecolor{mydarkblue}{rgb}{0,0.08,0.45}
\hypersetup{colorlinks=true,linkcolor=mydarkblue,citecolor=mydarkblue,
  filecolor=mydarkblue,urlcolor=mydarkblue}

\newcommand{\hdrcenter}[1]{\raisebox{\dimexpr-0.5\height+0.5\depth\relax}{#1}}
\fancypagestyle{titlelogo}{%
  \fancyhf{}%
  \fancyhead[L]{%
    \hdrcenter{\includegraphics[height=0.92cm]{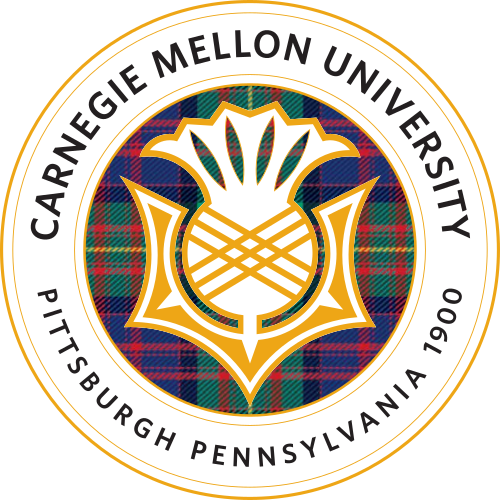}}%
    \hspace{2.2mm}%
    \hdrcenter{\fontsize{13}{15}\selectfont\bfseries Carnegie Mellon University}%
  }%
  \fancyhead[R]{\hdrcenter{\includegraphics[height=0.62cm]{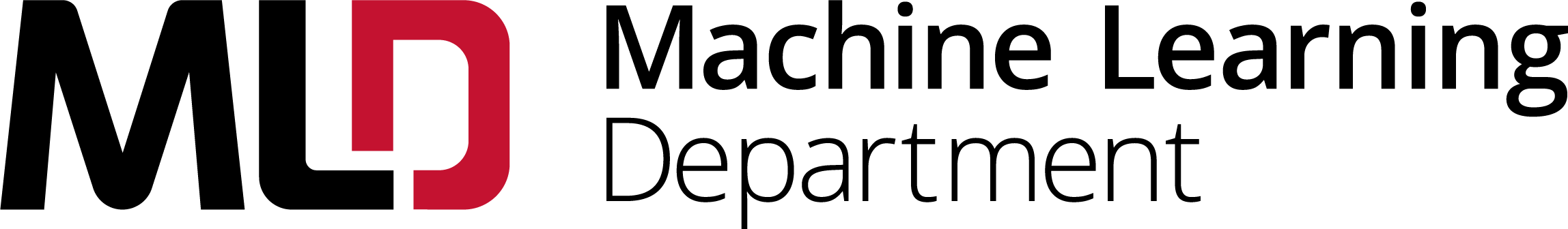}}}%
  \fancyfoot[C]{\thepage}%
}

\title{Momentum Streams for Optimizer-Inspired Transformers}

\author{
  Jingchu Gai\thanks{Authors are listed in alphabetical order.} \quad
  Nai-Chieh Huang\footnotemark[1] \quad Jiayun Wu\footnotemark[1] \\
  Carnegie Mellon University \\
  \texttt{jgai@andrew.cmu.edu} \quad \texttt{naichieh@andrew.cmu.edu} \quad
  \texttt{jiayunw@cmu.edu}
}

\begin{document}
\raggedbottom
% main.sty's `final' mode sets \pagestyle{empty}; restore page numbers.
\pagestyle{plain}
\maketitle
% Page-1 header band carrying the institutional logos.
\thispagestyle{titlelogo}

% \vspace{-10mm}

% \vspace{-5mm}
\vspace{-10mm}
\begin{figure}[!ht]
    \centering
    \includegraphics[width=15.5cm]{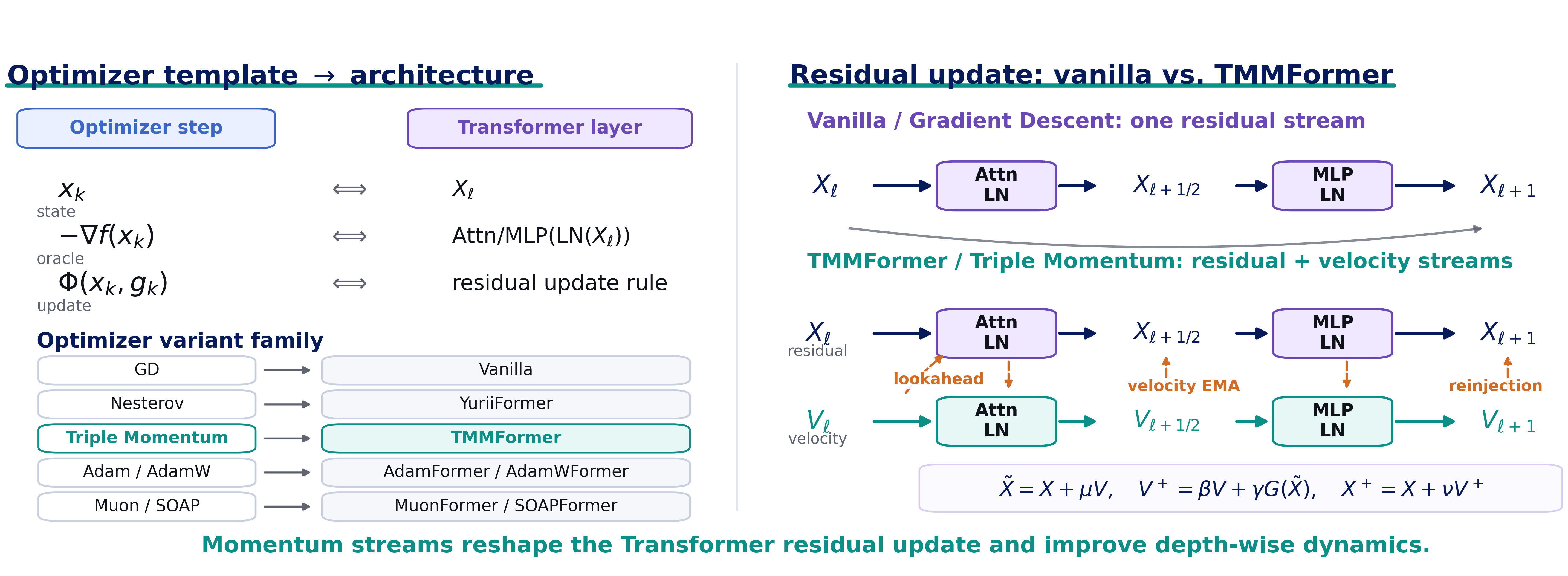}
    % \caption{Caption}
    \label{fig:placeholder}
\end{figure}
\begin{abstract}
The residual update of a pre-norm Transformer layer admits an interpretation as one step of a first-order optimizer acting on a surrogate token energy, wherein the attention and MLP sublayers function as gradient oracles.
Based on this observation, we build a family of
optimizer-inspired Transformers (triple-momentum, Adam/AdamW, Muon, SOAP) and
compare them under matched compute. In our main pretraining experiment, the triple-momentum TMMFormer achieves the lowest validation loss, outperforming the vanilla Transformer and prior architectural variants.
A controlled ablation and supporting theory show that momentum, not preconditioning, is the main source of the gain. We further show that
TMMFormer and other momentum-based designs reach flatter minima than the
vanilla Transformer, which leads to less forgetting and better generalization.

% We ask which optimizer the depth
% recurrence should implement and treat the optimizer template as an
% architectural design axis. Holding a single pre-norm attention/MLP backbone
% fixed and varying only the per-layer update rule, we build a family of
% optimizer-inspired Transformers (triple-momentum, Adam/AdamW, Muon, SOAP) and
% compare them under matched compute. The triple-momentum TMMFormer is the
% best validation-loss model, beating the vanilla stream and the prior YuriiFormer;
% the gain is architectural---robust to the training optimizer and learning
% rate---and is consistent with the local theory of triple-momentum depth
% recurrences. A controlled momentum$\times$preconditioning ablation
% attributes the improvement to the momentum stream rather than preconditioning,
% and a full-block Jacobian analysis shows it turns the residual update into a
% second-order depth filter that damps slow token-feature modes. The resulting
% models settle into flatter minima that forget less and generalize better out
% of distribution.
\end{abstract}

\begin{center}\small
\textbf{Code:}~\url{https://github.com/gaijingchu/Momentum-Streams-for-Optimizer-Inspired-Transformers}\\
\textbf{Checkpoints:}~\url{https://huggingface.co/gaijingchu/momentum-streams-checkpoints}
\end{center}

\setcounter{tocdepth}{2}

\input{body/introduction}
\input{body/related_work}
\input{body/preliminary}

\input{body/optimizer_inspired_transformer}
\input{body/momentum}
\input{body/properties_and_discussion}
\input{body/conclusion}

\bibliographystyle{acl_natbib}
\bibliography{custom}

\newpage
\tableofcontents
\newpage

\appendix

\input{appendix/optimizer_rules}
\input{appendix/extra_formers}
\input{appendix/matrix_preconditioning_design}
\input{appendix/experimental_details}
\input{appendix/precond_redundancy}
\input{appendix/momentum_theory}
\input{appendix/landscape_measurement}
\input{appendix/forgetting_generalization}
\input{appendix/sam_wsd}

\end{document}

%% file: preamble/_preamble_includes.tex
\input{preamble/header}

\input{preamble/theorem_style}

%% file: preamble/header.tex
\usepackage{geometry}
\usepackage{amsmath,amssymb,amsthm,mathtools,bm}
\usepackage{etoolbox}
\allowdisplaybreaks[4]
\DeclareMathSizes{11}{9}{7}{5}
\DeclareMathSizes{10.95}{9}{7}{5}
\usepackage{xcolor}
\usepackage[most]{tcolorbox}
\definecolor{skyblue}{HTML}{4A90D9}
\usepackage{booktabs}
\usepackage{enumitem}
\usepackage{hyperref}
\usepackage{capt-of}

\DeclareCaptionFont{figcap}{\fontsize{9pt}{11pt}\selectfont\itshape}
\makeatletter
\@ifpackageloaded{lineno}{%
  \usepackage{zref-savepos}%
  \newcounter{lnxpos}%
  \renewcommand\makeLineNumber{%
    \stepcounter{lnxpos}%
    \zsavepos{lnx\arabic{lnxpos}}%
    \ifdim\zposx{lnx\arabic{lnxpos}}sp>0.5\paperwidth
      \makeLineNumberRight
    \else
      \makeLineNumberLeft
    \fi}%
}{}
\makeatother

%% file: preamble/theorem_style.tex
\newtheorem{theorem}{Theorem}[section]
\newtheorem{lemma}[theorem]{Lemma}
\newtheorem{proposition}[theorem]{Proposition}
\newtheorem{corollary}[theorem]{Corollary}
\theoremstyle{remark}
\newtheorem{remark}[theorem]{Remark}

\tcbset{
  thmblue/.style={
    enhanced, breakable,
    colframe=blue!40!black, boxrule=0.35pt, arc=1mm,
    coltitle=black, fonttitle=\small\sffamily\bfseries,
    colbacktitle=blue!15!white, colback=blue!5!white,
    boxed title style={sharp corners, boxrule=0pt,
      top=1pt, bottom=0.5pt, left=4mm, right=4mm,
      borderline={0.5pt}{0pt}{blue!20!white}},
    attach boxed title to top left={xshift=4mm,yshift*=-1.2mm},
    boxsep=1.5mm, top=1.5mm, bottom=1.5mm, left=2.5mm, right=4mm,
    before skip=10pt, after skip=10pt,
    before upper={\let\linenomath\relax \let\endlinenomath\relax
      \let\linenomathAMS\relax}
  },
  thmmag/.style={
    enhanced, breakable,
    colframe=magenta!45!black, boxrule=0.35pt, arc=1mm,
    coltitle=black, fonttitle=\small\sffamily\bfseries,
    colbacktitle=magenta!15!white, colback=magenta!5!white,
    boxed title style={sharp corners, boxrule=0pt,
      top=1pt, bottom=0.5pt, left=4mm, right=4mm,
      borderline={0.5pt}{0pt}{magenta!25!white}},
    attach boxed title to top left={xshift=4mm,yshift*=-1.2mm},
    boxsep=1.5mm, top=1.5mm, bottom=1.5mm, left=2.5mm, right=4mm,
    before skip=10pt, after skip=10pt,
    before upper={\let\linenomath\relax \let\endlinenomath\relax
      \let\linenomathAMS\relax}
  }
}
\newtcolorbox{formerbox}[1]{
  enhanced jigsaw, breakable,
  pad at break*=0mm,
  colframe=skyblue, boxrule=0.6pt, arc=1mm,
  title={\textbf{#1}},
  coltitle=black, fonttitle=\small\sffamily\bfseries,
  colbacktitle=skyblue!20!white,
  colback=skyblue!5!white,
  boxed title style={sharp corners, boxrule=0pt, top=1pt, bottom=0.5pt,
    left=4mm, right=4mm,
    borderline={0.5pt}{0pt}{skyblue!20}},
  attach boxed title to top left={xshift=4mm,yshift*=-1.2mm},
  grow to left by=2.5mm, grow to right by=2.5mm,
  boxsep=0.6mm, top=2mm, bottom=2mm, left=0.8mm, right=0.8mm,
  before skip=8pt, after skip=8pt,
  parbox=false,
  before upper={\let\linenomath\relax \let\endlinenomath\relax
    \let\linenomathAMS\relax}
}

\BeforeBeginEnvironment{formerbox}{\nolinenumbers}
\AfterEndEnvironment{formerbox}{\linenumbers}

\newtcolorbox{tablebox}[1]{
  enhanced, hbox,
  title={#1},
  colback=teal!5, colframe=teal, coltext=black, coltitle=white,
  fonttitle=\small\bfseries, arc=1mm, boxrule=1pt,
  halign=center, fontupper=\normalsize,
  boxsep=1pt, left=2pt, right=2pt, top=2pt, bottom=2pt,
  toptitle=1pt, bottomtitle=1pt,
  before skip=6pt, after skip=4pt
}
\BeforeBeginEnvironment{tablebox}{\nolinenumbers\begin{center}}
\AfterEndEnvironment{tablebox}{\end{center}\linenumbers}

\newtcolorbox{methodbox}[1]{
  enhanced jigsaw, breakable,
  pad at break*=0mm,
  colframe=orange!75!white, boxrule=0.35pt, arc=1mm,
  title={\textbf{#1}},
  coltitle=black, fonttitle=\small\sffamily\bfseries,
  colbacktitle=orange!15!white,
  colback=orange!5!white,
  boxed title style={sharp corners, boxrule=0pt, top=1pt, bottom=0.5pt,
    left=4mm, right=4mm,
    borderline={0.5pt}{0pt}{orange!10}},
  attach boxed title to top left={xshift=4mm,yshift*=-1.2mm},
  boxsep=1.5mm, top=1mm, bottom=1mm, left=1mm, right=1mm,
  before skip=10pt, after skip=10pt,
  parbox=false,
  before upper={\let\linenomath\relax \let\endlinenomath\relax
    \let\linenomathAMS\relax}
}
\BeforeBeginEnvironment{methodbox}{\nolinenumbers}
\AfterEndEnvironment{methodbox}{\linenumbers}

\newtcolorbox{contribbox}[1]{
  enhanced jigsaw, breakable,
  pad at break*=0mm,
  colframe=green!40!black, boxrule=0.35pt, arc=1mm,
  title={\textbf{#1}},
  coltitle=black, fonttitle=\small\sffamily\bfseries,
  colbacktitle=green!15!white,
  colback=green!5!white,
  boxed title style={sharp corners, boxrule=0pt, top=1pt, bottom=0.5pt,
    left=4mm, right=4mm,
    borderline={0.5pt}{0pt}{green!20!white}},
  attach boxed title to top left={xshift=4mm,yshift*=-1.2mm},
  boxsep=1.5mm, top=1.5mm, bottom=1.5mm, left=1.5mm, right=1mm,
  before skip=10pt, after skip=10pt,
  parbox=false,
  before upper={\let\linenomath\relax \let\endlinenomath\relax
    \let\linenomathAMS\relax}
}
\BeforeBeginEnvironment{contribbox}{\nolinenumbers}
\AfterEndEnvironment{contribbox}{\linenumbers}

\DeclareMathOperator{\Attn}{Attn}
\DeclareMathOperator{\MLP}{MLP}
\DeclareMathOperator{\NS}{NS}

\DeclareMathOperator{\diag}{diag}
\DeclareMathOperator{\LN}{LN}
\DeclareMathOperator{\softplus}{softplus}
\DeclareMathOperator{\reshape}{reshape}

%% file: body/introduction.tex
\section{Introduction}

The Transformer has improved mostly by redesigning its attention and MLP
sublayers or by scaling them up. The \emph{depth recurrence}---how each layer
updates the residual stream---is almost always left as a plain additive
update and is rarely treated as a design choice in its own right. A recent
line of work suggests that a pre-norm Transformer layer can be
read as one step of a first-order optimizer on a surrogate token energy, with
attention and MLP acting as gradient
oracles~\citep{geshkovski2025mathematical,zimin2026yuriiformer}. Under this
correspondence the vanilla residual block is plain gradient descent, and the
recent YuriiFormer~\citep{zimin2026yuriiformer} is Nesterov momentum. This
invites a natural question: if a layer is an optimizer step, \emph{which}
optimizer should it implement, and how much does that choice matter? We
study a family of optimizer-inspired Transformers built from this
correspondence, and find that the choice matters substantially: stronger
momentum templates yield markedly better residual dynamics, and a
triple-momentum design, TMMFormer, is the strongest of the optimizers we try.

\textbf{Optimizer-inspired Transformer design.} Concretely, we treat the
optimizer template as an architectural axis: we hold the attention and MLP sublayers fixed and vary only the update rule of the residual stream.
Instantiating it with several classical optimizers yields a family of
optimizer-inspired Transformers---TMMFormer, AdamFormer, AdamWFormer,
MuonFormer, and SOAPFormer---which we pretrain under matched compute on
TinyStories and OpenWebText and evaluate for downstream transfer. TMMFormer
achieves the lowest validation loss in our main comparison, beating 
the vanilla residual stream and the prior YuriiFormer on both corpora. The advantage is architectural---robust to the parameter-training
optimizer and learning rate---and is consistent with the optimization
theory in which triple momentum is an optimal first-order method for
strongly convex quadratics
(Section~\ref{sec:optimizer-inspired-transformer}).

\textbf{Why momentum helps.} To explain the source of the gain,
we run a controlled momentum $\times$ preconditioning ablation. Adding a
momentum stream recovers most of the improvement, whereas diagonal
Adam-style preconditioning does not. The advantage
therefore comes from the momentum stream itself. A full-block Jacobian
analysis explains how: the auxiliary velocity turns the otherwise
first-order residual map into a second-order recurrence in depth---a filter
that changes how perturbations propagate through the residual stream. Across
variants, weaker minimum-gain persistence tracks validation loss, and a simple
local model shows the momentum recurrence is a strictly better forward filter
than the vanilla first-order stream (Section~\ref{sec:momentum}).

\textbf{Loss landscape and fine-tuning behavior.} To further explain the
gains of optimizer-inspired Transformers, we analyze their loss landscape and
fine-tuning behavior. The momentum variants converge to flatter minima than the
vanilla Transformer, and these flatter minima translate into less forgetting
after fine-tuning and better out-of-distribution generalization
(Section~\ref{sec:loss-landscape}).

% These dynamics leave a
% measurable trace on the loss surface: the optimizer-inspired variants---and
% the momentum ones most of all---settle into markedly flatter minima (lower
% Hessian trace and top eigenvalue) than the vanilla stream. Flatness in turn
% predicts behavior we care about: under sequential fine-tuning the flatter
% models forget less, and they generalize better to out-of-distribution
% corpora. A cheap warmup--stable--decay schedule, or SAM, flattens the minimum
% further, reinforcing the link. This gives a concrete reason why stronger
% momentum templates not only fit better but also transfer better
% (Section~\ref{sec:loss-landscape}).

% We find
% that it matters substantially---stronger momentum templates yield markedly
% better residual dynamics, and a triple-momentum design, TMMFormer, is the
% strongest of the optimizers we try.

\begin{contribbox}{Summary of Our Main Contribution}
\newcommand{\contribentry}[2]{%
  \par\noindent\hangindent=1.35em\hangafter=1%
  \makebox[1.35em][l]{#1.}#2\par}
\contribentry{1}{We study a family of optimizer-inspired Transformers and identify
the triple-momentum TMMFormer as the best validation-loss model in our main
comparison, beating both YuriiFormer and the vanilla Transformer.}
\contribentry{2}{With a controlled ablation and theory, we show that
momentum, not preconditioning, is the main source of the gain.}
\contribentry{3}{We show that the momentum variants reach flatter minima than
vanilla, with less forgetting and better out-of-distribution generalization.}
\end{contribbox}

%% file: body/related_work.tex
\section{Related Work}
\label{sec:related-work}

A line of work interprets deep learning architectures as discretizations of
continuous-time dynamics or as numerical schemes for evolving
representations. In the Transformer setting, \citet{lu2019understanding}
view the model as a numerical solver for a multi-particle dynamical
system and show that the standard pre-norm transformer block corresponds to a
first-order Lie--Trotter splitting between an inter-token interaction
term and a per-token potential term. This viewpoint also motivates
alternative splitting schemes such as the Strang--Marchuk-inspired
Macaron architecture, and provides early evidence that
numerical-analysis principles can guide Transformer design. Together,
this line of work argues that Transformer blocks are not immutable
design primitives, but can instead be derived from broader algorithmic
templates.

A second strand studies attention from a variational or
interacting-particle perspective.
\citet{geshkovski2025mathematical} provide a mathematical framework in
which Transformers are analyzed as interacting particle systems,
connecting attention dynamics to gradient flows and
long-time clustering over token configurations. This
viewpoint justifies treating attention not as a heuristic token-mixing
mechanism, but as implementing a structured gradient update on an
interaction energy. It provides the conceptual bridge to import
ideas of optimizer design into representation-space architecture rather
than only into parameter-space training.

The most directly related work is YuriiFormer
\citep{zimin2026yuriiformer}, which unifies attention and MLP updates
as gradient oracles for two complementary energies and reinterprets a
standard GPT-style block as gradient descent on the resulting composite
objective. Replacing this gradient-descent template with
Nesterov-accelerated momentum
\citep{nesterov269method,polyak1964some} yields a Transformer whose
residual stream is augmented by a velocity stream and whose forward
dynamics are second order in depth. Empirically, Nesterov momentum transformers improve validation loss  under matched parameter
and training budgets, with the improvements transferring to downstream
accuracy on reasoning tasks. YuriiFormer
thereby establishes both the conceptual validity and the practical
promise of optimizer-informed Transformer design.

% \paragraph{Optimizers beyond momentum, mostly studied in parameter space.}
% Modern training optimization has developed a much richer toolkit than
% plain gradient descent or Nesterov: adaptive moment methods such as
% Adam \citep{kingma2014adam} and its decoupled-weight-decay variant
% AdamW; matrix-orthogonalization-based methods such as Muon
% \citep{boreiko2025towards,bernstein2026muonpreconditioning}; and
% matrix-preconditioned methods such as Shampoo
% \citep{gupta2018shampoo} and SOAP \citep{vyas2024soap}. These methods
% estimate richer first- and second-moment or curvature structure than
% gradient descent and produce updates that are coordinate-, spectral-,
% or matrix-rescaled. However, they have been studied almost exclusively
% as \emph{parameter-space training optimizers}, not as
% \emph{representation-space architectural templates}.

\paragraph{Our extension.}
We extend the optimizer-informed Transformer design along two axes.
First, modern training optimization has developed a much richer toolkit than
plain gradient descent or Nesterov. We broaden the architectural template to cover
the Triple Momentum Method, AdamW, Muon, and SOAP, evaluated under
identical pre-norm backbones and matched budgets. Second, we investigate
\emph{which} aspect of an optimizer survives translation into the model architecture. Our experimental and theoretical results attribute the pretraining gains to the momentum stream, while diagonal and matrix preconditioning prove largely redundant with LayerNorm and attention.

%% file: body/preliminary.tex
\section{Preliminary}
\label{sec:preliminary}

\noindent\textbf{Notation.}
$X_\ell\in\mathbb{R}^{T\times d}$ is the residual stream at layer $\ell$ ($T$
tokens, model dimension $d$); $\ell+\tfrac12$ denotes the intermediate state
after the attention substep, so a layer factorizes as
$\ell\to\ell+\tfrac12\to\ell+1$ (attention then MLP). $\Attn_\ell$ and
$\MLP_\ell$ are the layer's two modules and $\LN(\cdot)$ the pre-norm
LayerNorm applied before each. Optimizer-inspired blocks also propagate one
or more auxiliary streams $\mathcal{S}_\ell$ (a velocity $V$, Adam moments
$M,S$, or a covariance $R$), stabilized by dedicated auxiliary LayerNorms
$\LN_v$ (velocity) and $\LN_u$ (update); auxiliary streams start from
separate learned token$+$position embeddings, except $S_0=\mathbf{1}$ and
$R_0=I_D$ for headwise channel covariances with head dimension $D=d/H$. Each
block carries a few learned per-layer scalars: we
suppress the layer index $\ell$ on them and use superscripts $(a)$ and $(m)$
for the attention- and MLP-substep copies, respectively
(e.g.\ $\mu^{(a)}$ vs.\ $\mu^{(m)}$). Scalars
constrained to $(0,1)$ are parameterized as $\sigma(\cdot)$ of an
unconstrained weight, positive scalars as $\softplus(\cdot)$.

\subsection{Optimization View of Transformers}
\label{sec:optim-transformer-v1}

Write $X_\ell=(x_1,\ldots,x_T)^\top$ for the residual stream at layer $\ell$,
with token rows $x_i\in\mathbb{R}^{d}$. We associate to $X$ a composite
surrogate energy
\begin{equation*}
    \mathcal{J}(X)
    \;=\;
    \mathcal{E}(X) + \mathcal{F}(X),
    % \label{eq:composite-v1}
\end{equation*}
where $\mathcal{E}(X) := \sum_{i,j} e^{\langle x_i, x_j \rangle}$ is a token--token \emph{interaction} energy and
$\mathcal{F}(X) := \sum_{i} U(x_i)$ is a per-token \emph{potential} energy. The two
sublayers of a Transformer block then act as learned negative-gradient oracles
\citep{geshkovski2025mathematical,zimin2026yuriiformer}:
\begin{align*}
    \Attn_\ell\!\bigl(X\bigr)
    &\;\approx\;
    -\nabla \mathcal{E}_\ell(X),\quad
    \mathrm{MLP}_\ell\!\bigl(X\bigr)
    \;\approx\;
    -\nabla \mathcal{F}_\ell(X).
\end{align*}
Any first-order descent method can be formulated as the template $x_{k+1}
    \;=\;
    \Phi\bigl(x_k,\, d_k\bigr)$, $d_k \approx -\nabla f(x_k).$
% \begin{equation*}
%     x_{k+1}
%     \;=\;
%     \Phi\bigl(x_k,\, d_k\bigr),
%     \qquad d_k \approx -\nabla f(x_k).
%     % \label{eq:opt-template-v1}
% \end{equation*}
The optimizer template can be lifted to a pre-norm Transformer block by identifying
$x_k \leftrightarrow X_\ell$, replacing the descent direction $d_k$ by the two
negative-gradient oracles $\Attn_\ell,\mathrm{MLP}_\ell$ applied to $\LN(X)$, and
discretizing the joint flow on $\mathcal{E}+\mathcal{F}$ with an
Lie--Trotter splitting so that the two oracles are
applied sequentially.
Concretely, a single Transformer layer becomes
\begin{align}
\begin{split}
\label{eq:template_opt}
        X_{\ell+1/2}
    &=
    \mathrm{Opt}_\ell
    \bigl(X_\ell,\; \mathcal{S}_\ell;\;
    \Attn_\ell(\LN(\cdot))\bigr),
    \\
    X_{\ell+1}
    &=
    \mathrm{Opt}_\ell
    \bigl(X_{\ell+1/2},\; \mathcal{S}_{\ell+1/2};\;
    \mathrm{MLP}_\ell(\LN(\cdot))\bigr),
\end{split}
\end{align}
where $\mathrm{Opt}_\ell$
is the optimizer step written in terms of a descent direction, $\mathcal{S}_\ell$
denotes any auxiliary state the optimizer maintains alongside the
iterate, such as moments. Different choices of
$\mathrm{Opt}_\ell$ thus give rise to different
\emph{optimizer-informed Transformers} sharing the same
$\Attn_\ell,\mathrm{MLP}_\ell$ backbone.
The vanilla pre-norm transformer block is the Lie--Trotter splitting of gradient
descent:
\begin{equation*}
\begin{aligned}
    X_{\ell+1/2}
    &= X_\ell + \Attn_\ell(\LN(X_\ell)),
    \\
    X_{\ell+1}
    &= X_{\ell+1/2}
       + \mathrm{MLP}_\ell(\LN(X_{\ell+1/2})).
\end{aligned}
% \label{eq:vanilla-block-v1}
\end{equation*}

\subsection{Optimizers Beyond Gradient Descent}
\label{sec:optimizers-v1}

We will instantiate the template~\eqref{eq:template_opt} with five
optimizer families that span the principal axes of modern first-order
optimization. We give only the
essential structure here. Full update rules are in
Appendix~\ref{sec:optimizer-rules}.

\paragraph{Heavy-Ball and Nesterov momentum.}
Polyak's heavy-ball method \citep{polyak1964some} augments gradient
descent with a velocity buffer $v_k$ that filters past gradients.
Nesterov's accelerated gradient \citep{nesterov269method} additionally
evaluates the gradient at a \emph{lookahead point}
$x_k + \mu_k v_k$ and, on the strongly convex quadratic model with curvature
range $[m,L]$, has contraction factor
$(\sqrt{L}-\sqrt{m})/(\sqrt{L}+\sqrt{m})$.
YuriiFormer~\citep{zimin2026yuriiformer} applies the Nesterov method to template~\eqref{eq:template_opt}, leading to the attention layer
\begin{align*}
    X^{\mathrm{in}}_\ell
    &= X_\ell + \mu\, V_\ell,
    \\
    V_{\ell+1/2}
    &= \LN_v\!\Bigl(
         \beta\, V_\ell + \gamma\,
         \Attn_\ell\!\bigl(\LN(X^{\mathrm{in}}_\ell)\bigr)
       \Bigr),
    \\
    X_{\ell+1/2}
    &= X_\ell + V_{\ell+1/2},
\end{align*}
where $\mu,\beta,\gamma$ are learned scalars and the auxiliary state is the
velocity buffer, $\mathcal{S}_\ell = V_\ell$.

\paragraph{Triple Momentum Method (TMM).}
TMM extends Nesterov with a second scalar $\nu_k$ that decouples the
gradient-evaluation lookahead from the iterate update,
$x_{k+1} = x_k + \nu_k v_{k+1}$. This gives a larger second-order update
family; in the local analysis below, setting $\nu_k\equiv1$ recovers the
YuriiFormer/Nesterov-style update.

\paragraph{Adam and AdamW.}
Adam \citep{kingma2014adam} maintains EMAs of the first and second
moments of the gradient, and rescales each coordinate to obtain a coordinate-wise adaptive step size.
AdamW decouples weight decay from the adaptive update. Both methods
produce updates that are \emph{element-wise} rescaled.

\paragraph{Muon.}
Muon \citep{boreiko2025towards,bernstein2026muonpreconditioning} treats a
matrix-valued parameter as a single object and performs steepest
descent under the spectral norm. It maintains a momentum
buffer and replaces it by its orthogonal polar factor through
Newton--Schulz iterations, producing updates whose singular values are
all approximately one. This makes the update \emph{spectrally
isotropic} rather than coordinate-wise.

\paragraph{Shampoo and SOAP.}
Shampoo \citep{gupta2018shampoo} maintains separate left and right
Gram matrix accumulators $L_t, R_t$ for a matrix gradient $G\in
\mathbb{R}^{m\times n}$ and applies a Kronecker-factored preconditioner
$L_t^{-1/4}\,G\,R_t^{-1/4}$. SOAP \citep{vyas2024soap} improves and
stabilizes Shampoo by performing Adam-style moment
updates in the eigenbasis of the Kronecker factors. Both combine matrix
preconditioning and adaptive updates.

%% file: body/optimizer_inspired_transformer.tex
\section{Optimizer Inspired Transformer}
\label{sec:optimizer-inspired-transformer}

In this section we describe the design of three optimizer-inspired
transformer variants, and then present our main evaluation results and ablation
studies. Building on the optimizer--as--architecture correspondence in
Section~\ref{sec:preliminary}, we instantiate TMMFormer, AdamFormer, and
MuonFormer, and study their relative behavior under matched compute. We also
build a SOAPFormer variant (right-factor Kronecker preconditioning), an
AdamWFormer variant that augments AdamFormer with decoupled weight decay, and
a factorial sweep of additional optimizer cells (Heavy-Ball, RMSProp,
orthogonal, Shampoo). To keep the main text focused, the update rules for
these six variants are deferred to Appendix~\ref{sec:extra-formers}.

\subsection{Optimizer Inspired Transformer Design}
\label{sec:transformer-design}

% \paragraph{TMMFormer (Triple Momentum Method).}
% A velocity stream $V_\ell\in\mathbb{R}^{T\times d}$ is propagated alongside
% the residual state $X_\ell$, with eight learned scalars per layer: a lookahead
% coefficient $\mu$, a velocity decay $\beta$, an oracle-output gain $\gamma$,
% and a velocity-reinjection gain $\nu$ (an attention and an MLP copy of each,
% superscripts $(a)$ and $(m)$).
\paragraph{TMMFormer (Triple Momentum Method).} TMMFormer propagates a velocity stream $V_\ell\in\mathbb{R}^{T\times d}$ alongside the residual state $X_\ell$. Each layer learns four scalars for each substep---lookahead $\mu$, velocity decay $\beta$, oracle gain $\gamma$, and reinjection gain $\nu$---with separate attention and MLP copies $(a)$ and $(m)$.

\begin{formerbox}{TMMFormer}
The attention substep
$\ell\to\ell+1/2$ is
\begin{align*}
    &\textstyle \widetilde{X}_\ell = X_\ell + \mu^{(a)}\,V_\ell,
\\
    &\textstyle V_{\ell+1/2} = \LN_v(
        \beta^{(a)}\,V_\ell
        + \gamma^{(a)}\,\Attn_\ell(\LN(\widetilde{X}_\ell))
    ),
\\
    &\textstyle X_{\ell+1/2} = X_\ell + \nu^{(a)}\,V_{\ell+1/2},
\end{align*}
and the MLP substep $\ell+1/2\to\ell+1$ is
\begin{align*}
    &\textstyle \widetilde{X}_{\ell+1/2} = X_{\ell+1/2} + \mu^{(m)}\,V_{\ell+1/2},
\\
    &\textstyle V_{\ell+1} = \LN_v(
        \beta^{(m)}\,V_{\ell+1/2}
        + \gamma^{(m)}\,\MLP_\ell(\LN(\widetilde{X}_{\ell+1/2}))
    ),
\\
    &\textstyle X_{\ell+1} = X_{\ell+1/2} + \nu^{(m)}\,V_{\ell+1}.
\end{align*}
\end{formerbox}

Each substep applies a lookahead $\widetilde{X}=X+\mu V$, a velocity EMA
(old velocity decayed by $\beta$, fresh oracle output scaled by $\gamma$,
renormalized by $\LN_v$), and an iterate update that moves $X$ along the new
velocity with gain $\nu$. Setting $\nu\equiv 1$ recovers
YuriiFormer~\citep{zimin2026yuriiformer}, which is our initialization for
$\nu$ (Appendix~\ref{sec:appendix-arch-constants}).

%\vspace{-8pt}
\paragraph{AdamFormer.}
Two auxiliary streams $(M_\ell,S_\ell)\in\mathbb{R}^{T\times d}\times
\mathbb{R}^{T\times d}_{>0}$ track per-token first and second moments of the
oracle output. Six learned scalars per layer: first- and second-moment decays
$\beta_1,\beta_2$ and an update gain $\gamma$ (an attention and an MLP copy of
each).

The oracle is queried at the current state $X$ (no lookahead). $M$ and $S$
are EMAs of the oracle output and of its element-wise square, and $X$ is
updated along the Adam direction $M/(\sqrt{S}+\varepsilon)$, renormalized by
$\LN_u$ and scaled by $\gamma$. AdamWFormer adds decoupled weight decay
(Appendix~\ref{sec:adamw-former}).

\begin{formerbox}{AdamFormer}
The attention substep is
\begin{align*}
    &\textstyle G_\ell = \Attn_\ell(\LN(X_\ell)),
\\
    &\textstyle M_{\ell+1/2} = \beta_1^{(a)}\,M_\ell + (1-\beta_1^{(a)})\,G_\ell,
\\
    &\textstyle S_{\ell+1/2} = \beta_2^{(a)}\,S_\ell
       + (1-\beta_2^{(a)})\,G_\ell\odot G_\ell,
\\
    &\textstyle X_{\ell+1/2} = X_\ell
       + \gamma^{(a)}\,\LN_u(
           \tfrac{M_{\ell+1/2}}
                {\sqrt{S_{\ell+1/2}}+\varepsilon}
       ),
\end{align*}
and the MLP substep is analogous with $\MLP_\ell$ replacing $\Attn_\ell$ on
input $X_{\ell+1/2}$:
\begin{align*}
    &\textstyle G_{\ell+1/2} = \MLP_\ell(\LN(X_{\ell+1/2})),
\\
    &\textstyle M_{\ell+1} = \beta_1^{(m)}\,M_{\ell+1/2}
       + (1-\beta_1^{(m)})\,G_{\ell+1/2},
\\
    &\textstyle S_{\ell+1} = \beta_2^{(m)}\,S_{\ell+1/2}
       + (1-\beta_2^{(m)})\,G_{\ell+1/2}\odot G_{\ell+1/2},
\\
    &\textstyle X_{\ell+1} = X_{\ell+1/2}
       + \gamma^{(m)}\,\LN_u(
           \tfrac{M_{\ell+1}}
                {\sqrt{S_{\ell+1}}+\varepsilon}
       ).
\end{align*}
\end{formerbox}

% \paragraph{MuonFormer (orthogonalized momentum).}
% A single momentum stream $M_\ell\in\mathbb{R}^{T\times d}$ is propagated; the
% update is orthogonalized by a per-token, head-wise Newton--Schulz operator
% $\NS(\cdot):\mathbb{R}^{T\times d}\to\mathbb{R}^{T\times d}$ before being added
% back into the residual (per-token reshape sizes and iteration count in
% Appendix~\ref{sec:appendix-arch-constants}). Four learned scalars per layer: a
% momentum decay $\beta$ and an update gain $\gamma$ (an attention and an MLP
% copy of each).

\paragraph{MuonFormer (orthogonalized momentum).}
MuonFormer propagates a momentum stream $M_\ell\in\mathbb{R}^{T\times d}$ and
orthogonalizes its update with a per-token, head-wise Newton--Schulz operator
$\NS(\cdot)$ before residual addition (reshape sizes and iterations in
Appendix~\ref{sec:appendix-arch-constants}). Each layer learns decay $\beta$
and gain $\gamma$ for both attention and MLP substeps.

\begin{formerbox}{MuonFormer}
The
attention substep is
\begin{align*}
    &\textstyle G_\ell = \Attn_\ell(\LN(X_\ell)),
\\
    &\textstyle M_{\ell+1/2} = \beta^{(a)}\,M_\ell + (1-\beta^{(a)})\,G_\ell,
\\
    &\textstyle X_{\ell+1/2} = X_\ell + \gamma^{(a)}\,\LN_u(
        \NS(M_{\ell+1/2})
    ),
\end{align*}
and the MLP substep is
\begin{align*}
    &\textstyle G_{\ell+1/2} = \MLP_\ell(\LN(X_{\ell+1/2})),
\\
    &\textstyle M_{\ell+1} = \beta^{(m)}\,M_{\ell+1/2}
       + (1-\beta^{(m)})\,G_{\ell+1/2},
\\
    &\textstyle X_{\ell+1} = X_{\ell+1/2} + \gamma^{(m)}\,\LN_u(
        \NS(M_{\ell+1})
    ).
\end{align*}
\end{formerbox}

$M$ is a momentum EMA of the oracle output. The update uses its orthogonal polar factor, computed per token and per head by the
Newton--Schulz operator $\NS$ (driving singular values to $1$)~\citep{boreiko2025towards}, which preserves causality for autoregressive training.

\subsection{Experimental Results}
\label{sec:optimizer-experiments}

\paragraph{Experimental Setup.}
\label{sec:experimental-setup}
All variants, including the vanilla baseline of
Section~\ref{sec:preliminary}, share an identical
$12$-layer pre-norm Transformer backbone with $12$ attention heads per layer
and model dimension $d=768$ (context $1024$, GPT-2 BPE, weight-tied head) and
differ \emph{only} in the optimizer template of
Section~\ref{sec:transformer-design}: vanilla has $124$M parameters and
every auxiliary-stream variant $\approx 163$M. We pretrain from scratch on
TinyStories (TS, $10$k steps) and OpenWebText (OWT, $30$k steps) with an
effective batch of $480$ sequences and a warmup-then-cosine schedule.
Parameters are trained with two coupled optimizers---Muon on the $2$D
weight matrices and AdamW on embeddings, LayerNorms, and the learned
per-layer scalars (at a higher learning rate)---so that the training
optimizer and the architectural template stay strictly separate. We report
best validation cross-entropy (nats/token) and downstream
\texttt{acc\_norm} on HellaSwag and ARC-Easy. Full hyperparameters are
listed in Appendix~\ref{sec:appendix-hparam-table}
(Tables~\ref{tab:hp-backbone}--\ref{tab:hp-system}).

\begin{tablebox}{Main results}
\setlength{\tabcolsep}{4.5pt}
\renewcommand{\arraystretch}{1.15}
\begin{tabular}{@{}lcccc@{}}
& \multicolumn{2}{c}{\textbf{val loss} $\downarrow$}
& \multicolumn{2}{c}{\textbf{acc\_norm} (\%) $\uparrow$} \\
\cmidrule(lr){2-3}\cmidrule(lr){4-5}
\textbf{Variant} & TS & OWT & HS & ARC \\
\midrule
VanillaTransformer & $1.1569$ & $3.0078$ & $30.20$ & $41.67$ \\
AdamFormer         & $1.1528$ & $2.9911$ & $30.96$ & $43.39$ \\
AdamWFormer        & $1.1472$ & $2.9883$ & $30.08$ & $41.88$ \\
YuriiFormer        & $1.1317$ & $2.9413$ & $31.58$ & $43.06$ \\
\textbf{TMMFormer} & $\mathbf{1.1284}$ & $\mathbf{2.9342}$
                   & $\mathbf{31.82}$  & $\mathbf{43.43}$ \\
\end{tabular}
\end{tablebox}
\captionof{table}{Best val loss (nats/token) and OWT downstream
\texttt{acc\_norm} (\%); full results in
Appendix~\ref{sec:detailed-results}.}
\label{tab:main-results}

\begin{figure}[t]
\centering
% \vspace{-20pt}
\includegraphics[width=\textwidth]{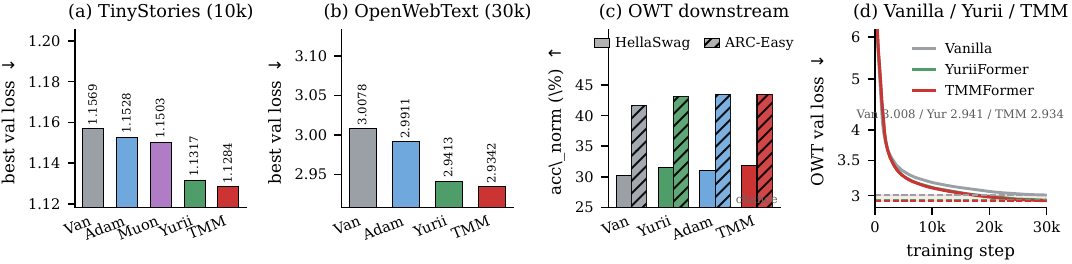}
% \vspace{-25pt}
\caption{Matched-compute comparison. \textbf{(a,b)} best val loss
(TinyStories, OpenWebText); \textbf{(c)} OWT downstream \texttt{acc\_norm};
\textbf{(d)} Vanilla, YuriiFormer, and TMMFormer OWT val-loss curves.}
\label{fig:results-overview}
% \vspace{-5pt}
\end{figure}

% \vspace{-12pt}
\paragraph{Main Results.}
\textbf{TMMFormer gives the best validation loss in our main comparison}
(Table~\ref{tab:main-results}, Figure~\ref{fig:results-overview}): it attains
the lowest validation loss on both corpora ($1.1284$ TS, $2.9342$ OWT) and the
best downstream accuracy on HellaSwag and ARC-Easy, improving on vanilla by
${\approx}0.029$ nats on TS and ${\approx}0.074$ on OWT. Among the rows
reported in Table~\ref{tab:main-results}, it is best in every column, and in
particular surpasses
YuriiFormer~\citep{zimin2026yuriiformer}, the prior momentum-stream
architecture that it generalizes (the $\nu\equiv1$ special case). This is
consistent with the optimization view of Section~\ref{sec:preliminary}: on the
quadratic local model used to analyze the depth recurrence, triple momentum
has a stronger first-order rate than the Nesterov and gradient-descent
templates. TMMFormer also does not materially increase runtime: it and the
vanilla transformer both train at ${\approx}3$\,s/step on OpenWebText (no
significant difference), with full per-variant wall-clock in
Appendix~\ref{sec:appendix-wall-clock}. Among the preconditioned variants,
MuonFormer and SOAPFormer
converge but trail every momentum-stream design---best TinyStories validation
loss $1.1503$ (Muon) and $1.1431$ (SOAP), and on OpenWebText MuonFormer
reaches only $3.0096$, no better than the vanilla stream---consistent with
preconditioning being a weak inductive bias for the residual stream
(Section~\ref{sec:momentum}). Full numbers, training curves, and analysis are
in Appendix~\ref{sec:detailed-results}.

\paragraph{Ablation Study.}
To confirm the advantage is architectural, we run four ablations.

$\bullet$ \emph{Peak learning rate.} We halve the Muon peak learning rate
    ($4{\times}10^{-3}\!\to\!2{\times}10^{-3}$) with the rest of the recipe
    held fixed. On OpenWebText, Vanilla degrades from $3.008$ to $3.029$
    ($+0.021$ nats) and TMMFormer from $2.934$ to $2.963$ ($+0.029$); the
    architectural gap shrinks only ${\approx}11\%$ (from $0.074$ to $0.066$),
    so the ordering is robust to LR perturbation.
    
$\bullet$ \emph{Parameter-training optimizer.} We swap the optimizer on the
    2D weight matrices (default Muon hybrid $\to$ a single pure AdamW for
    every parameter group), forming a full $2{\times}2$ with the two
    architectures. The pure-AdamW peak LR is the standard
    nanoGPT~/~GPT-2-small default ($6{\times}10^{-4}$), not retuned per
    architecture, so the comparison uses a community-accepted recipe for both
    sides. On OpenWebText, Vanilla barely moves ($3.008\!\to\!3.010$, $+0.002$
    nats), while TMMFormer moves more ($2.934\!\to\!2.970$, $+0.035$);
    TMMFormer still beats Vanilla under both optimizers (gap $0.074$ under
    Muon hybrid, $0.041$ under pure AdamW).
    
$\bullet$ \emph{Parameter-matched controls.} We train a Vanilla variant with the
    same parameter count as TMMFormer ($\approx163$M, achieved by widening
    $d_{\mathrm{model}}$ to $900$). On TinyStories it closes only about $40\%$ of
    the TMM--Vanilla gap (best val $1.1454$, vs.\ TMMFormer's $1.1272$ and the
    default Vanilla's $1.1578$); the remaining ${\approx}0.018$ nats is roughly
    $9\times$ the per-seed standard deviation, so the bulk of TMMFormer's gain is
    not a parameter-count effect.
    
$\bullet$ \emph{Multi-seed validation.} Separate three-seed TinyStories runs
    ($10$k steps) give Vanilla $1.1578\pm0.0028$ and TMMFormer
    $1.1272\pm0.0013$; these mean/std values need not match the single-run
    values in Table~\ref{tab:main-results}. The gap of $0.0306$ is
    ${\approx}15\times$ the pooled seed standard deviation, so it is not seed
    luck. We run this on TinyStories because the smaller
    corpus has both a tighter architectural gap and less data, so seed noise
    could plausibly be larger; since the noise on TS is already very small,
    OWT (with a larger gap) should be at least as stable.
    Across all four ablations the noise is small (a few thousandths of a nat) and
    the architectural effect is the dominant signal; full per-ablation numbers
    are in Appendix~\ref{sec:detailed-results}.

%% file: body/momentum.tex
\section{Momentum}
\label{sec:momentum}

The previous section suggests that momentum is the useful ingredient:
YuriiFormer~\citep{zimin2026yuriiformer} and TMMFormer outperform the other
optimizer-inspired variants, but they also include lookahead or learned
velocity reinjection. To isolate momentum from preconditioning, we introduce
three controlled variants: HBFormer keeps only a heavy-ball velocity stream,
RMSPropFormer keeps only diagonal second-moment preconditioning, and
OrthoFormer keeps only the Muon-style spectral preconditioner. Together with
Vanilla, AdamFormer, and MuonFormer, these variants form a
momentum$\times$preconditioning ablation.

\begin{tablebox}{Momentum $\times$ preconditioning (OWT val loss)}
\setlength{\tabcolsep}{6pt}
\renewcommand{\arraystretch}{1.15}
\begin{tabular}{lccc}
 & No precond. & Diag. precond. & Spectral precond. \\
\hline
No momentum & Vanilla: $3.008$ & RMSProp: $3.015$ & Ortho: $3.033$ \\
Momentum & HB: $2.945$ & AdamFormer: $2.991$ & Muon: $3.010$ \\
\end{tabular}
\end{tablebox}
\captionof{table}{Momentum--preconditioning ablation.
AdamFormer isolates diagonal preconditioning without AdamW weight
decay; OrthoFormer is MuonFormer without momentum.}
\label{tab:momentum-two-by-two}

\subsection{Momentum Is Sufficient for Most of the Gain}
\label{sec:momentum-sufficient}

The cleanest evidence comes from the no-preconditioning column of
Table~\ref{tab:momentum-two-by-two}. Adding only a heavy-ball momentum stream
improves OpenWebText validation loss from $3.008$ to $2.945$. This comparison
does not add Adam-style second moments, orthogonalization, or matrix
preconditioning, so it isolates the effect of momentum itself.

The broader no-preconditioning ordering is then:
$3.008\;(\text{Vanilla})
    >
    2.945\;(\text{HB}) 
    >
    2.941\;(\text{Yurii})
    >
    2.934\;(\text{TMM}),$
% \begin{align*}
%    \textstyle  3.008\;(\text{Vanilla})
%     &>
%     2.945\;(\text{HB}) \notag\\
%     &>
%     2.941\;(\text{Yurii})
%     >
%     2.934\;(\text{TMM}),
% \end{align*}
where lower validation loss is better. A minimal heavy-ball stream already
recovers most of the vanilla-to-Yurii gap; lookahead and learned reinjection
then add smaller gains, with TMMFormer lowest in this comparison.

\subsection{Preconditioning Does Not Explain the Improvement}
\label{sec:preconditioning-not-source}

The same ablation shows that preconditioning---diagonal or spectral---does not
explain the gains in this setting. Without momentum, RMSPropFormer is slightly
worse than Vanilla ($3.015$ versus $3.008$) and OrthoFormer worse still
($3.033$). With momentum, AdamFormer ($2.991$) and MuonFormer ($3.010$) are
both worse than the plain heavy-ball HBFormer ($2.945$), so in these runs the
preconditioned momentum variants underperform the unpreconditioned momentum
variant. The momentum effect, by contrast, holds in every column---each
momentum row beats its no-momentum counterpart. Thus the main benefit comes
from momentum rather than from adaptive or spectral rescaling of token-space
update directions.

For diagonal preconditioning, Appendix~\ref{sec:preconditioning-redundancy}
gives a supporting explanation, summarized informally below.
\begin{tcolorbox}[thmblue, title={\textbf{Diagonal preconditioning is redundant}}]
\begin{theorem}
\label{thm:informal-diagonal-redundancy}
If the Adam- or RMSProp-style second-moment stream is approximately
coordinate-balanced, then its diagonal preconditioner is close to a scalar
multiple of the identity. Because AdamFormer and RMSPropFormer apply an update
LayerNorm after this preconditioned update, the scalar component is absorbed
and cannot create a new residual direction.
\end{theorem}
\end{tcolorbox}
Thus, in the diagonal case, any useful preconditioning effect must come from
non-scalar deviations that survive the update LayerNorm.

\subsection{What Does Momentum Change?}
\label{sec:what-momentum-changes}

\begin{figure}[t]
\centering
% \vspace{-15pt}
\includegraphics[width=\textwidth]{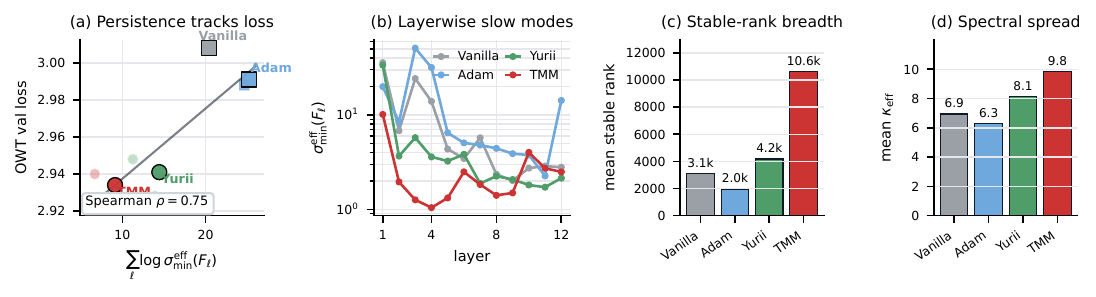}
\vspace{-20pt}
\caption{Full-block Jacobian spectra: (a) minimum-gain persistence versus loss,
(b) layerwise minimum-gain persistence, (c) stable rank, and (d) spectral
spread.}
\label{fig:momentum-landscape}
% \vspace{-5pt}
\end{figure}

We next ask what momentum changes inside the forward computation. Our main
finding is that momentum changes how perturbations propagate across depth:
momentum variants have lower minimum-gain persistence than Vanilla while
preserving a broad transition spectrum. This is a forward-propagation
diagnostic, not an optimization condition-number claim. To measure it, we
analyze the full layer transition, not the raw attention/MLP oracle, because
momentum blocks also include auxiliary streams, update LayerNorms, and learned
scalar gates. The diagnostic below summarizes the full-block Jacobian along
trained trajectories.

\begin{methodbox}{Full-Block Jacobian Diagnostic}
For each layer, let $F_\ell:X_\ell\mapsto X_{\ell+1}$ be the complete block
map, including auxiliary streams, update LayerNorms, and learned scalar gates,
and let
\begin{equation*}
J_\ell=\partial F_\ell/\partial X_\ell\,\big|_{X_\ell=\bar X_\ell}
\end{equation*}
be its Jacobian at trained activations, with auxiliary streams fixed to their
trajectory values. We write $\sigma_{\max}^{\mathrm{eff}}(J_\ell)$ and
$\sigma_{\min}^{\mathrm{eff}}(J_\ell)$ for the largest and smallest singular
values kept by the same fixed numerical cutoff across all variants. From these
singular values, we report:
\begin{align*}
\text{min-gain persist.:}\quad
\mathcal{P}
&=\textstyle\sum_{\ell}\log\sigma_{\min}^{\mathrm{eff}}(J_\ell),
\\
\text{stable rank:}\quad
r_{\mathrm{st}}(J_\ell)
&=\|J_\ell\|_F^2/\|J_\ell\|_2^2,
\\
\text{spread:}\quad
\kappa_{\mathrm{eff}}(J_\ell)
&=\sigma_{\max}^{\mathrm{eff}}(J_\ell)/
   \sigma_{\min}^{\mathrm{eff}}(J_\ell).
\end{align*}
\emph{Minimum-gain persistence} sums the weakest layerwise gain across depth;
lower values mean less amplification along the most contracted measured
directions. \emph{Stable rank} measures how broadly the Jacobian spectrum is
used; higher values mean a less concentrated transition. \emph{Spread} is the
ratio between the largest and smallest effective gains; higher values mean a
wider range of forward amplification, not a worse optimizer condition number.
\end{methodbox}

Figure~\ref{fig:momentum-landscape} summarizes the result. Across the analyzed
OpenWebText variants, lower minimum-gain persistence tracks lower validation
loss: the momentum variants move to the low-persistence, low-loss region,
whereas AdamFormer does not. At the same time, the momentum variants have
higher stable rank and larger spectral spread, indicating that they do not
simply collapse the transition spectrum. Instead, momentum reshapes the
full-block Jacobian into a broader forward filter while reducing the
persistence of the most contracted measured directions. Together, these results
suggest that momentum improves the residual stream by changing how perturbations
propagate across layers, rather than merely changing the scale of individual
updates.

Appendix~\ref{sec:momentum-residual-filtering} gives complementary theoretical
support for this mechanism in a simplified local model; we summarize the
implication informally here.
\begin{tcolorbox}[thmblue, title={\textbf{Momentum yields second-order filtering}}]
\begin{theorem}
\label{thm:informal-momentum-filter}
In a local linearized residual-stream model around a task-relevant hidden
representation, a vanilla residual Transformer implements a first-order
polynomial filter over token-feature modes, whereas a momentum-stream
Transformer implements a second-order filter. For a nontrivial local spectrum
with condition number $\kappa>1$, there exist stable momentum coefficients
whose worst-case contraction factor is
$(\sqrt{\kappa}-1)/(\sqrt{\kappa}+1)$, strictly smaller than the best
fixed-step first-order factor $(\kappa-1)/(\kappa+1)$.
\end{theorem}
\end{tcolorbox}

This result is local, not a global language-modeling guarantee, but it explains
why an auxiliary velocity gives a richer finite-depth filter than the
vanilla first-order stream. Appendix \ref{sec:momentum-residual-filtering} also shows that TMMFormer
contains YuriiFormer as the $\nu_\ell=1$ special case, supporting the empirical
ordering Vanilla $<$ HBFormer $<$ YuriiFormer $<$ TMMFormer.

%% file: body/properties_and_discussion.tex
\section{Loss Landscape of Optimizer Based Transformer}
\label{sec:loss-landscape}
In this section, we study the loss-landscape sharpness of the
optimizer-inspired transformers introduced in
Section~\ref{sec:optimizer-inspired-transformer}. We first measure the loss
landscape of the vanilla transformer and of the optimizer-inspired variants,
and show that the latter exhibit a flatter landscape than the vanilla
transformer. We then measure the forgetting and generalization behavior of
the momentum variants, and show that they generalize better and forget less
than the vanilla transformer. Finally, we
show that adding a learning-rate schedule further improves TMMFormer, while
additionally mitigating forgetting and improving generalization. Together,
these results indicate that a flatter loss landscape is one of the reasons the
optimizer-inspired transformers outperform the vanilla transformer.

\begin{figure}[t]
\centering
% \vspace{-15pt}
\includegraphics[width=\textwidth]{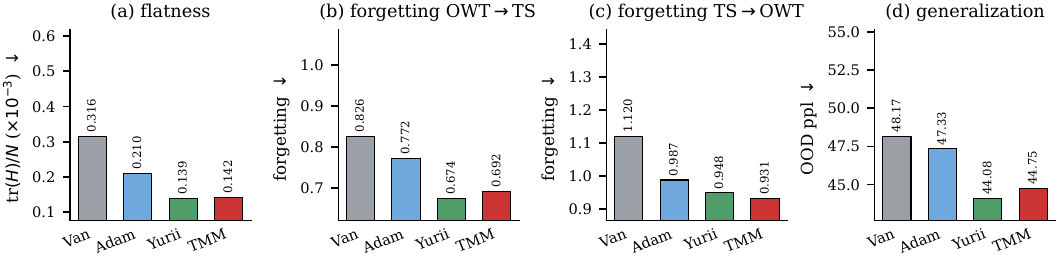}
\vspace{-10pt}
\caption{Flatness, forgetting, and generalization (lower is better).
\textbf{(a)} $\operatorname{tr}(H)/N$ ($\times10^{-3}$); \textbf{(b,c)}
forgetting (OWT$\to$TS, TS$\to$OWT); \textbf{(d)} out-of-distribution
perplexity.}
\label{fig:properties-overview}
\end{figure}

\subsection{Loss Landscape}

\paragraph{Setup.}
For each variant, we probe the curvature of the validation cross-entropy loss
around the trained parameters $\theta$ at its best checkpoint, using three
standard sharpness diagnostics: the top Hessian eigenvalue $\lambda_{\max}$
(via power iteration on Hessian--vector products), the Hessian trace
$\operatorname{tr}(H)$ (via the Hutchinson estimator), and the loss range
along a filter-normalized one-dimensional perturbation~\citep{li2018visualizing},
together with the scale-invariant trace $\operatorname{tr}(H)/N$. Every variant
is probed on the same fixed validation batches, and a lower value of each
diagnostic indicates a flatter minimum. Full estimator definitions and
hyperparameters are given in Appendix~\ref{sec:landscape-measurement}.

\paragraph{Results.}
Figure~\ref{fig:properties-overview}(a) reports the parameter-normalized
Hessian trace $\operatorname{tr}(H)/N$ ($\times10^{-3}$). The vanilla
transformer sits in the sharpest minimum ($0.316$). AdamFormer already
flattens it substantially ($0.210$), and the momentum variants reach the
flattest minima by a wide margin---$0.139$ for YuriiFormer and $0.142$ for
TMMFormer, roughly $2.2\times$ flatter than vanilla and about a third lower
than AdamFormer. YuriiFormer is slightly flatter than TMMFormer by this
diagnostic, even though TMMFormer has the lower validation loss. Thus the
landscape measurements should be read as supporting the broader momentum
effect---optimizer-inspired gains are accompanied by substantially flatter
minima than vanilla---rather than as a strict ordering among the two best
momentum variants.

\subsection{Forgetting and Generalization}
The flatter minima found above are classically linked to better
generalization and greater robustness to distribution
shift~\citep{foret2020sharpness}. We test whether this link holds for the
optimizer-inspired variants by measuring how much each one forgets under
sequential fine-tuning and how well it generalizes out of distribution.
\paragraph{Setup.}
We probe two complementary axes. For \emph{forgetting}, we fine-tune each
pretrained checkpoint on a second corpus and measure how much it degrades on
its original one: from an OpenWebText (resp.\ TinyStories) model we fine-tune
on TinyStories (resp.\ OpenWebText) and report \emph{forgetting}, the rise in
the original corpus's loss. To isolate the architecture, every variant is fine-tuned with
the \emph{same} AdamW optimizer and schedule, regardless of its pretraining
optimizer. For \emph{generalization}, we evaluate the OpenWebText checkpoint
zero-shot (no fine-tuning) on three out-of-distribution corpora---WikiText-103,
LAMBADA, and C4---and report the average perplexity. Lower forgetting and
lower out-of-distribution perplexity are better. Full fine-tuning
hyperparameters, the evaluation protocol, and corpus details are given in
Appendix~\ref{sec:forgetting-generalization-measurement}.

\paragraph{Results.}
Both axes follow the broad flatness pattern of
Figure~\ref{fig:properties-overview}(b--d). Forgetting decreases
monotonically from the vanilla transformer to the momentum variants in
\emph{both} transfer directions: for OpenWebText$\to$TinyStories it drops
from $0.83$ (vanilla) to $0.77$ (AdamFormer) to $0.69$ (TMMFormer) and
$0.67$ (YuriiFormer), and TinyStories$\to$OpenWebText shows the same ordering
($1.12\to0.99\to0.93$/$0.95$). Out-of-distribution generalization shows the
same broader pattern: the
average perplexity over WikiText-103, LAMBADA, and C4 falls from $48.2$ for
vanilla and $47.3$ for AdamFormer to $44.8$ for TMMFormer and $44.1$ for
YuriiFormer. YuriiFormer is slightly better on these robustness diagnostics,
whereas TMMFormer is better on validation loss. The main conclusion is that
the momentum variants sit in flatter minima, forget less, and transfer better
than vanilla. Per-corpus numbers are reported in
Appendix~\ref{sec:forgetting-generalization-measurement}.

\subsection{Learning-Rate Schedule and Sharpness-Aware Minimization}

We test two low-cost training interventions on TMMFormer, changing only the
parameter-training recipe (the architecture is unchanged): a
warmup--stable--decay (WSD) learning-rate schedule in place of the default
warmup--cosine, and Sharpness-Aware Minimization~\citep{foret2020sharpness}
(SAM), which takes an extra ascent step toward the worst-case loss in a
$\rho$-ball before each update. We also combine them (\emph{SAWD}: the WSD
schedule with SAM applied only during the decay phase). Schedule and
SAM details are in Appendix~\ref{sec:sam-wsd-measurement}.

Both interventions flatten the TMMFormer minimum. WSD lowers
$\operatorname{tr}(H)/N$ from $0.142$ to $0.106\times10^{-3}$ and improves OWT validation loss from $2.934$ to $2.924$. SAM gives a faltter minimum ($0.062\times10^{-3}$, ${\approx}2.3\times$ below cosine) but does not improve validation loss ($2.934\!\to\!2.940$).

%% file: body/conclusion.tex
\section{Conclusion}
Viewing a pre-norm Transformer layer as one step of a first-order optimizer
on a surrogate token energy turns the choice of optimizer into an
architectural design axis. Among the resulting optimizer-inspired
Transformers, TMMFormer---the triple-momentum template---achieves the lowest
validation loss in our main comparison, beating the vanilla stream and the
prior YuriiFormer on both corpora; this gain is robust to the training
optimizer and learning rate, and is consistent with the local theory of
triple-momentum depth recurrences. Controlled experiments further show that
the improvement comes from the optimizer's momentum design, while its
preconditioning design does not explain the gain. Finally, analyzing the loss
landscape and fine-tuning behavior of optimizer-inspired Transformers, we
find that the momentum variants reach flatter minima than vanilla, which in
turn yields less forgetting and better generalization; YuriiFormer is
slightly better than TMMFormer on several of these robustness diagnostics, so
these results should be interpreted as evidence for the broader momentum
effect rather than a strict win for TMMFormer on every metric.

% We further shows that optimizer based transformer yielding flatter minima that forget less and generalize
% better. Our models are ${\sim}160$M parameters at modest budgets and the
% filtering analysis is local rather than a global guarantee; scaling and
% tightening the theory are natural next steps.

%% file: appendix/optimizer_rules.tex
\section{Optimizer Update Rules}
\label{sec:optimizer-rules}

This appendix collects the parameter-space update rules for the
optimizer templates summarised in Section~\ref{sec:preliminary}. Each
template lifts to a Transformer block by the recipe of that section:
replace $\nabla f$ by the two oracles $\Attn_\ell(\LN(\cdot))$ and
$\mathrm{MLP}_\ell(\LN(\cdot))$, maintain the optimizer's auxiliary
state as a parallel residual stream, and learn the per-layer scalars.

\subsection{Heavy-Ball and Nesterov Accelerated Gradient}
\label{sec:hb-nag}

Polyak's heavy-ball method \citep{polyak1964some} adds a momentum
buffer $v_k$ to gradient descent. Nesterov's accelerated gradient
(NAG) \citep{nesterov269method} evaluates the gradient at a lookahead
point. With state $x_k$, velocity $v_k$, lookahead $\mu_k$, momentum
$\beta_k\in(0,1)$, and step size $\gamma_k>0$:
\begin{align*}
    \tilde{x}_k
    &= x_k + \mu_k\, v_k,
    \\
    v_{k+1}
    &= \beta_k\, v_k - \gamma_k\, \nabla f(\tilde{x}_k),
    \\
    x_{k+1}
    &= x_k + v_{k+1}.
\end{align*}
Heavy-ball is the special case $\mu_k=0$. For the strongly convex quadratic
model with curvature range $[m,L]$, NAG has contraction factor
$(\sqrt{L}-\sqrt{m})/(\sqrt{L}+\sqrt{m})$.

\subsection{Triple Momentum Method}
\label{sec:tmm}

The Triple Momentum Method (TMM) introduces a second scalar $\nu_k$
that decouples the iterate update from the gradient lookahead:
\begin{align*}
    \tilde{x}_k
    &= x_k + \mu_k\, v_k,
    \\
    v_{k+1}
    &= \beta_k\, v_k - \gamma_k\, \nabla f(\tilde{x}_k),
    \\
    x_{k+1}
    &= x_k + \nu_k\, v_{k+1}.
\end{align*}
NAG corresponds to $\nu_k \equiv 1$. In the layer-indexed lift, this containment
is the main architectural point: the TMM template can recover the
Nesterov/YuriiFormer update while learning a larger second-order filter class.

\subsection{Adam and AdamW}
\label{sec:adam-adamw}

Adam \citep{kingma2014adam} maintains exponential moving averages of
the first and second moments of the gradient
$g_k=\nabla f(x_k)$:
\begin{align*}
    m_{k}
    &= \beta_{1}\, m_{k-1} + (1-\beta_{1})\, g_k,
    \\
    s_{k}
    &= \beta_{2}\, s_{k-1} + (1-\beta_{2})\, g_k\odot g_k,
    \\
    \hat{m}_{k} &\textstyle= m_{k} / (1-\beta_1^{k}), \; \hat{s}_{k} = s_{k} / (1-\beta_2^{k}),\\
    x_{k+1}
    &= x_k - \gamma_k
        \frac{m_{k}}{\sqrt{s_{k}}+\varepsilon}.
\end{align*}
Here $\beta_{1},\beta_{2}\in[0,1)$ are decay rates and
$\varepsilon>0$ is a numerical floor. Bias correction by
$1-\beta_{j}^{k}$ is omitted in the layer-indexed lift to AdamFormer.

AdamW differs from Adam by decoupling weight decay $\lambda_k\in
(0,1)$ from the adaptive update, shrinking the iterate before
applying the Adam step:
\begin{equation*}
    x_{k+1}
    \;=\;
    (1-\lambda_k \gamma_k)\, x_k
    \;-\;
    \gamma_k\,
    \frac{m_{k}}{\sqrt{s_{k}}+\varepsilon}.
\end{equation*}

\subsection{Muon }
\label{sec:muon}

For a matrix-valued parameter $W \in \mathbb{R}^{m\times n}$ with
gradient $G_k=\nabla_W f(W_k)$, Muon
\citep{boreiko2025towards,bernstein2026muonpreconditioning} replaces the
momentum buffer by its orthogonal polar factor before applying it:
\begin{align*}
    G^{m}_{k+1}
    &= \beta_k\, G^{m}_{k} + (1-\beta_k)\, G_k,
    \\
    U_{k+1}
    &= \mathrm{NS}_K\!\bigl(G^{m}_{k+1}\bigr),
    \\
    W_{k+1}
    &= W_k - \gamma_k\, U_{k+1},
\end{align*}
where $\mathrm{NS}_K(\cdot)$ denotes $K$ steps of the quintic
Newton--Schulz iteration applied to the Frobenius-normalized buffer
$Y_0 = G^{m}_{k+1}/\|G^{m}_{k+1}\|_F$:
\begin{equation*}
    Y_{j+1}
    \;=\;
    Y_j\bigl(a\, I + b\, Y_j^\top Y_j + c\, (Y_j^\top Y_j)^2\bigr).
\end{equation*}
$Y_K$ approximates the polar factor of
$G^{m}_{k+1}$, so all singular values of $U_{k+1}$ are approximately
one. The resulting update is steepest descent under the spectral
norm.

\subsection{Shampoo and SOAP}
\label{sec:shampoo-soap}

For a matrix iterate $W \in \mathbb{R}^{m\times n}$ with gradient
$G_k$, Shampoo \citep{gupta2018shampoo} maintains Kronecker-factored
gram matrix accumulators
\begin{align*}
    L_{k+1}
    &= L_k + G_k G_k^\top
    \;\in\; \mathbb{R}^{m\times m},
    \\
    R_{k+1}
    &= R_k + G_k^\top G_k
    \;\in\; \mathbb{R}^{n\times n},
    \\
    W_{k+1}
    &= W_k - \gamma_k\, L_{k+1}^{-1/4}\, G_k\, R_{k+1}^{-1/4}.
\end{align*}
The matrix inverse fourth roots realize a structured preconditioner
of the tensor.

SOAP \citep{vyas2024soap} improves and stabilizes Shampoo by carrying
out Adam-style first- and second-moment updates in the eigenbasis of
the Kronecker factors. Letting $L_{k+1}=Q_L\Lambda_L Q_L^\top$ and
$R_{k+1}=Q_R\Lambda_R Q_R^\top$, SOAP rotates the gradient as
$\widetilde{G}_k = Q_L^\top G_k Q_R$, runs Adam in the rotated
coordinate system on $\widetilde{G}_k$, and rotates the resulting
update back. Compared to plain Shampoo, this adds per-direction
adaptivity at modest extra cost. 

A common simplification used in
practice, and adopted in our Shampoo/SOAP-inspired Transformer variant, is to
keep only the right factor $R_{k+1}$.

%% file: appendix/extra_formers.tex
\section{Additional Optimizer-Inspired Transformer Variants}
\label{sec:extra-formers}

We collect here the update rules for the optimizer-inspired transformer
variants whose details we deferred from
Section~\ref{sec:optimizer-inspired-transformer}: AdamWFormer (which adds
decoupled weight decay on top of AdamFormer) and four \emph{factorial
ablation} cells---HBFormer, RMSPropFormer, OrthoFormer, ShampooFormer---each
of which keeps exactly one of \{momentum, preconditioner\} and removes the
other from one of the four main-text architectures. The notation
$\ell\to\ell+1/2\to\ell+1$, the $(a)$ and $(m)$ superscripts on substep scalars,
and the auxiliary LayerNorms $\LN_v,\LN_u$ are as in
Section~\ref{sec:transformer-design}.

\subsection{AdamWFormer}
\label{sec:adamw-former}

AdamWFormer keeps the AdamFormer first/second-moment streams
$(M_\ell,S_\ell)$ and adds two scalars per substep,
$\lambda^{(a)},\lambda^{(m)}\in(0,1)$, implementing the AdamW-style decoupled
weight decay: before the adaptive update is applied, the residual is
contracted by $1-\lambda$. Eight learned scalars per layer
$(\beta_1^{(a)},\beta_2^{(a)},\gamma^{(a)},\lambda^{(a)},\beta_1^{(m)},\beta_2^{(m)},
\gamma^{(m)},\lambda^{(m)})$.

\begin{formerbox}{AdamWFormer}
The attention substep is
\begin{align*}
    &\textstyle G_\ell = \Attn_\ell(\LN(X_\ell)),
\\
    &\textstyle M_{\ell+1/2} = \beta_1^{(a)}\,M_\ell + (1-\beta_1^{(a)})\,G_\ell,
\\
    &\textstyle S_{\ell+1/2} = \beta_2^{(a)}\,S_\ell + (1-\beta_2^{(a)})\,G_\ell\odot G_\ell,
\\
    &\textstyle X_{\ell+1/2} = (1-\lambda^{(a)})\,X_\ell
       + \gamma^{(a)}\,\LN_u(
           \tfrac{M_{\ell+1/2}}
                {\sqrt{S_{\ell+1/2}}+\varepsilon}
       ),
\end{align*}
and the MLP substep is
\begin{align*}
    &\textstyle G_{\ell+1/2} = \MLP_\ell(\LN(X_{\ell+1/2})),
\\
    &\textstyle M_{\ell+1} = \beta_1^{(m)}\,M_{\ell+1/2}
       + (1-\beta_1^{(m)})\,G_{\ell+1/2},
\\
    &\textstyle S_{\ell+1} = \beta_2^{(m)}\,S_{\ell+1/2}
       + (1-\beta_2^{(m)})\,G_{\ell+1/2}\odot G_{\ell+1/2},
\\
    &\textstyle X_{\ell+1} = (1-\lambda^{(m)})\,X_{\ell+1/2}
       + \gamma^{(m)}\,\LN_u(
           \tfrac{M_{\ell+1}}
                {\sqrt{S_{\ell+1}}+\varepsilon}
       ).
\end{align*}
\end{formerbox}

We initialize $\lambda^{\mathrm{raw}}=-5$ so
$\sigma(\lambda^{\mathrm{raw}})\approx 0.007$: the model starts with
near-trivial decay and learns where to push it up. Setting
$\lambda^{(a)}\equiv\lambda^{(m)}\equiv 0$ recovers AdamFormer of
Section~\ref{sec:transformer-design}.

\subsection{SOAPFormer (Right-Factor Kronecker Preconditioning)}
\label{sec:soap-former}

A first-moment stream $M_\ell\in\mathbb{R}^{T\times H\times D}$ in head space
and a per-token right covariance $R_\ell\in\mathbb{R}^{T\times D\times D}$ are
propagated alongside the residual; the update is preconditioned by $R^{-1/2}$
on the channel side. Six learned scalars per layer: a first-moment decay
$\beta_1$, a covariance decay $\beta_R$, and an update gain $\gamma$ (an
attention and an MLP copy of each). Below $G_\ell\in\mathbb{R}^{T\times
H\times D}$ is the oracle output viewed as a per-token $(H,D)$ matrix, and
$\reshape:\mathbb{R}^{T\times H\times D}\to\mathbb{R}^{T\times d}$ collapses
the head and head-dim axes back to the residual-stream layout.

\begin{formerbox}{SOAPFormer}
The attention substep is
\begin{align*}
    &\textstyle G_\ell = \Attn_\ell(\LN(X_\ell)),
\\
    &\textstyle M_{\ell+1/2} = \beta_1^{(a)}\,M_\ell + (1-\beta_1^{(a)})\,G_\ell,\\
    &\textstyle R_{\ell+1/2} = \beta_R^{(a)}\,R_\ell
       + (1-\beta_R^{(a)})\,G_\ell^\top G_\ell,
\\
    &\textstyle X_{\ell+1/2} = X_\ell + \gamma^{(a)}\,\LN_u(
        \reshape(M_{\ell+1/2}\,R_{\ell+1/2}^{-1/2})
    ),
\end{align*}
and the MLP substep is
\begin{align*}
    &\textstyle G_{\ell+1/2} = \MLP_\ell(\LN(X_{\ell+1/2})),
\\
    &\textstyle M_{\ell+1} = \beta_1^{(m)}\,M_{\ell+1/2}
       + (1-\beta_1^{(m)})\,G_{\ell+1/2},
\\
    &\textstyle R_{\ell+1} = \beta_R^{(m)}\,R_{\ell+1/2}
       + (1-\beta_R^{(m)})\,G_{\ell+1/2}^\top G_{\ell+1/2},
\\
    &\textstyle X_{\ell+1} = X_{\ell+1/2} + \gamma^{(m)}\,\LN_u(
        \reshape(M_{\ell+1}\,R_{\ell+1}^{-1/2})
    ).
\end{align*}
\end{formerbox}

$M$ is an Adam-style first-moment EMA of $G$ and $R$ is an EMA of the
per-token outer product $G^\top G$; the update $M R^{-1/2}$ rescales each
token's head matrix by its accumulated channel covariance. We drop the SOAP
left covariance because it would precondition the token axis, where attention
already learns token mixing; keeping only the right/channel covariance also
avoids an additional token-side matrix factor (Appendix~\ref{sec:token-side-redundancy}).
$R^{-1/2}$ is computed by $K=10$ Newton iterations on matmuls; $R_0=I_D$ per
token, where $D=d/H$ is the head dimension. Although the identity initialization
keeps the EMA covariance positive definite, each new outer-product update has
rank at most $H=12$ in $D=64$, making the estimate poorly conditioned in this
low-rank regime. We therefore avoid eigendecomposition; empirically SOAPFormer
fails to converge (Section~\ref{sec:optimizer-experiments}).

\subsection{Factorial-Ablation Variants}
\label{sec:factorial-formers}

The four core architectures (TMMFormer, AdamFormer, MuonFormer, SOAPFormer)
populate the \emph{momentum}$\times$%
\emph{preconditioner} table at the cells (Nesterov+TMM, none),
(heavy-ball, per-coord), (heavy-ball, spectral), (heavy-ball, full-matrix).
The four variants below populate four additional cells of the same table,
each obtained by either (i) keeping momentum but removing the preconditioner
(HBFormer) or (ii) keeping a preconditioner but removing the momentum
(RMSPropFormer, OrthoFormer, ShampooFormer). All four share the
attention/MLP Lie--Trotter splitting and the auxiliary-LayerNorm convention
of the main text; only their auxiliary streams and update directions
change.

\subsubsection{HBFormer (Heavy-Ball Momentum, No Preconditioner)}
\label{sec:hb-former}

HBFormer keeps the velocity stream of TMM/YuriiFormer but removes lookahead and
fixes the $\nu$ iterate-gain scalar to $\nu\equiv1$---i.e., classical Polyak
heavy-ball. A velocity stream $V_\ell\in\mathbb{R}^{T\times d}$ is
propagated, with four learned scalars per layer
$(\beta^{(a)},\gamma^{(a)},\beta^{(m)},\gamma^{(m)})$.

\begin{formerbox}{HBFormer}
The attention substep is
\begin{align*}
    &\textstyle G_\ell = \Attn_\ell(\LN(X_\ell)),
\\
    &\textstyle V_{\ell+1/2} = \LN_v(\beta^{(a)}\,V_\ell + \gamma^{(a)}\,G_\ell),
\\
    &\textstyle X_{\ell+1/2} = X_\ell + V_{\ell+1/2},
\end{align*}
and the MLP substep is
\begin{align*}
    &\textstyle G_{\ell+1/2} = \MLP_\ell(\LN(X_{\ell+1/2})),
\\
    &\textstyle V_{\ell+1} = \LN_v(\beta^{(m)}\,V_{\ell+1/2}
       + \gamma^{(m)}\,G_{\ell+1/2}),
\\
    &\textstyle X_{\ell+1} = X_{\ell+1/2} + V_{\ell+1}.
\end{align*}
\end{formerbox}

HBFormer is the special case of YuriiFormer with
$\mu^{(a)}\equiv\mu^{(m)}\equiv 0$ (no lookahead), and the special case of
TMMFormer with $\mu^{(a)}\equiv\mu^{(m)}\equiv 0$ and
$\nu^{(a)}\equiv\nu^{(m)}\equiv 1$. It isolates the contribution of pure momentum to the
residual stream---no gradient lookahead, no learnable iterate gain.

\subsubsection{RMSPropFormer (Per-Coordinate Preconditioner, No Momentum)}
\label{sec:rmsprop-former}

RMSPropFormer keeps AdamFormer's second-moment stream $S_\ell$ and removes
the first-moment stream $M_\ell$---i.e., the update direction is the raw
oracle divided by $\sqrt{S}$. A single auxiliary stream
$S_\ell\in\mathbb{R}^{T\times d}_{>0}$ is propagated, with four learned
scalars per layer $(\beta_2^{(a)},\gamma^{(a)},\beta_2^{(m)},\gamma^{(m)})$.

\begin{formerbox}{RMSPropFormer}
The attention
substep is
\begin{align*}
    &\textstyle G_\ell = \Attn_\ell(\LN(X_\ell)),
\\
    &\textstyle S_{\ell+1/2} = \beta_2^{(a)}\,S_\ell + (1-\beta_2^{(a)})\,G_\ell\odot G_\ell,
\\
    &\textstyle X_{\ell+1/2} = X_\ell + \gamma^{(a)}\,\LN_u(
        \tfrac{G_\ell}{\sqrt{S_{\ell+1/2}}+\varepsilon}
    ),
\end{align*}
and the MLP substep is
\begin{align*}
    &\textstyle G_{\ell+1/2} = \MLP_\ell(\LN(X_{\ell+1/2})),
\\
    &\textstyle S_{\ell+1} = \beta_2^{(m)}\,S_{\ell+1/2}
       + (1-\beta_2^{(m)})\,G_{\ell+1/2}\odot G_{\ell+1/2},
\\
    &\textstyle X_{\ell+1} = X_{\ell+1/2} + \gamma^{(m)}\,\LN_u(
        \tfrac{G_{\ell+1/2}}{\sqrt{S_{\ell+1}}+\varepsilon}
    ).
\end{align*}
\end{formerbox}

$S_0=\mathbf{1}$ as in AdamFormer. RMSPropFormer is the special case of
AdamFormer with $\beta_1^{(a)}\equiv\beta_1^{(m)}\equiv 0$, so that the first moment $M$
collapses to the raw oracle $G$. It isolates the contribution of
coordinate-wise second-moment preconditioning, without any first-moment
smoothing.

\subsubsection{OrthoFormer (Spectral Preconditioner, No Momentum)}
\label{sec:ortho-former}

OrthoFormer is MuonFormer with the momentum EMA removed: the oracle output
is orthogonalized directly. There is no auxiliary stream; only two learned
scalars per layer $(\gamma^{(a)},\gamma^{(m)})$.

\begin{formerbox}{OrthoFormer}
The attention substep is
\begin{align*}
    &\textstyle G_\ell = \Attn_\ell(\LN(X_\ell)),
\\
    &\textstyle X_{\ell+1/2} = X_\ell + \gamma^{(a)}\,\LN_u(\NS(G_\ell)),
\end{align*}
and the MLP substep is
\begin{align*}
    &\textstyle G_{\ell+1/2} = \MLP_\ell(\LN(X_{\ell+1/2})),
\\
    &\textstyle X_{\ell+1} = X_{\ell+1/2} + \gamma^{(m)}\,\LN_u(\NS(G_{\ell+1/2})).
\end{align*}
\end{formerbox}

$\NS$ denotes the per-token, head-wise Newton--Schulz polar-factor operator
defined in Section~\ref{sec:transformer-design}. OrthoFormer is the special
case of MuonFormer with $\beta^{(a)}\equiv\beta^{(m)}\equiv 0$ (no EMA, fresh oracle every
step). It isolates the contribution of the spectral / isotropy bias of
Newton--Schulz, separate from any momentum-style smoothing of the update.

\subsubsection{ShampooFormer (Full-Matrix Preconditioner, No Momentum)}
\label{sec:shampoo-former}

ShampooFormer is SOAPFormer with the first-moment stream removed: the raw
oracle is preconditioned by the running channel-side covariance. A single
per-token right covariance $R_\ell\in\mathbb{R}^{T\times D\times D}$ is
propagated, with four learned scalars per layer
$(\beta_R^{(a)},\gamma^{(a)},\beta_R^{(m)},\gamma^{(m)})$. With $G_\ell$ the per-token
$(H,D)$ reshape of the oracle output and $\reshape$ the inverse of that
reshape (as in SOAPFormer, Appendix~\ref{sec:soap-former}),

\begin{formerbox}{ShampooFormer}
the attention substep
is
\begin{align*}
    &\textstyle G_\ell = \Attn_\ell(\LN(X_\ell)),
\\
    &\textstyle R_{\ell+1/2} = \beta_R^{(a)}\,R_\ell + (1-\beta_R^{(a)})\,G_\ell^\top G_\ell,
\\
    &\textstyle X_{\ell+1/2} = X_\ell + \gamma^{(a)}\,\LN_u(
        \reshape(G_\ell\,R_{\ell+1/2}^{-1/2})
    ),
\end{align*}
and the MLP substep is
\begin{align*}
    &\textstyle G_{\ell+1/2} = \MLP_\ell(\LN(X_{\ell+1/2})),
\\
    &\textstyle R_{\ell+1} = \beta_R^{(m)}\,R_{\ell+1/2}
       + (1-\beta_R^{(m)})\,G_{\ell+1/2}^\top G_{\ell+1/2},
\\
    &\textstyle X_{\ell+1} = X_{\ell+1/2} + \gamma^{(m)}\,\LN_u(
        \reshape(G_{\ell+1/2}\,R_{\ell+1}^{-1/2})
    ).
\end{align*}
\end{formerbox}

$R_0=I_D$ per token; $R^{-1/2}$ is computed by Newton iterations on matmuls
as in SOAPFormer. ShampooFormer is the special case of SOAPFormer with
$\beta_1^{(a)}\equiv\beta_1^{(m)}\equiv 0$, so the first moment $M$ collapses to the raw
oracle $G$. It isolates the contribution of full-matrix channel-side
preconditioning, without first-moment smoothing.

\paragraph{Summary.}
Together with the three main-text variants, the six additional variants of
this appendix fill in the rest of the (momentum$\,\times\,$preconditioner)
factorial table that we use to attribute the generalization gap between
TMM/Yurii and Adam(W) to its momentum vs.\ preconditioner components. The
results corresponding to each cell are reported in
Section~\ref{sec:optimizer-experiments} (factorial ablation).

%% file: appendix/matrix_preconditioning_design.tex
\section{Token-Side Redundancy in Matrix Preconditioning}
\label{sec:token-side-redundancy}

SOAP- or Shampoo-style matrix preconditioning often uses both a token-side
factor and a channel-side factor. SOAPFormer deliberately drops the token-side
factor and keeps only the right, channel-side covariance. The reason is that
the token-side factor acts on the same axis as attention: it mixes token
positions. The next proposition gives the linearized motivation for this design
choice.

\begin{tcolorbox}[thmmag, title={\textbf{Token-side preconditioning Redundancy}}]
\begin{proposition}
\label{prop:token-side-preconditioning-redundancy}
Consider a linearized attention oracle
\begin{equation*}
    \mathrm{Attn}(X)
    =
    A X W,
\end{equation*}
where
$X \in \mathbb{R}^{T \times d}$,
$A \in \mathbb{R}^{T \times T}$ is a token-mixing matrix,
and $W \in \mathbb{R}^{d \times d}$ is a channel map.
Let a token-side preconditioner act as
\begin{equation*}
    G
    \mapsto
    P G,
    \qquad
    P = L^{-1/2}
    \in \mathbb{R}^{T \times T}.
\end{equation*}
Then there exists another linear token-mixing matrix
\begin{equation*}
    \widetilde{A}
    =
    P A
\end{equation*}
such that
\begin{equation*}
    P\,\mathrm{Attn}(X)
    =
    \widetilde{A} X W .
\end{equation*}
Thus, in the linearized regime, a token-side preconditioner can be represented
as a replacement of the attention mixing matrix by another linear token-mixing
matrix.
\end{proposition}
\end{tcolorbox}

\begin{proof}
By direct substitution,
\begin{align*}
    P\,\mathrm{Attn}(X)
    &=
    P A X W .
\end{align*}
Defining $\widetilde{A}=PA$ gives
\begin{equation*}
    P\,\mathrm{Attn}(X)
    =
    \widetilde{A} X W .
\end{equation*}
Therefore the preconditioned operation has the form of another linear
token-mixing map.
\end{proof}

\begin{remark}[Interpretation]
\label{rem:soapformer-right-factor}
The proposition is a linearized representational statement. The matrix
$\widetilde{A}=PA$ need not be a valid softmax attention matrix: it may fail to
be nonnegative, row-stochastic, or causal. Thus the result should not be read as
an exact identity inside standard softmax attention, nor as a proof that
token-side preconditioning is algebraically unnecessary in the full model.
Rather, it motivates why a token-side SOAP/Shampoo factor is less compelling as
an architectural intervention: its role overlaps with attention's learned token
mixing, while the right factor acts on channel correlations that attention does
not directly precondition.
\end{remark}

%% file: appendix/experimental_details.tex
\section{Experimental Details}
\label{sec:experimental-details}

This appendix gives the full training and evaluation configuration used for
every optimizer-inspired Transformer variant in
Section~\ref{sec:optimizer-experiments}. Unless noted otherwise, all
variants share every value in Tables~\ref{tab:hp-backbone}--\ref{tab:hp-system};
they differ only
in the optimizer template of Section~\ref{sec:transformer-design} and in the
auxiliary streams it propagates. All numbers are taken directly from the
training code.

\subsection{Backbone and Tokenization}
\label{sec:appendix-backbone}

Every model uses a $12$-layer, $12$-head, $d=768$ pre-norm Transformer with
a context length of $1024$ tokens. Tokenization is the GPT-2 byte-pair
encoding (\texttt{tiktoken} \texttt{gpt2}, $|\mathcal V|=50{,}304$), and the
output projection is weight-tied to the token embedding. The vanilla model
has $124$M parameters; every auxiliary-stream variant has $\approx 163$M,
the extra $\approx 39$M coming entirely from the separate learned
token+position embeddings that initialize the auxiliary stream(s)
($V_0$, $M_0$, $\ldots$). The per-layer learned scalars
$\omega_\ell$ are a negligible parameter count ($\le 8$ per layer) but are
trained on a separate, higher learning rate (see below).

\subsection{Architectural Constants and Initialization}
\label{sec:appendix-arch-constants}

The per-layer scalars are parameterized as in the Notation paragraph of
Section~\ref{sec:transformer-design}: scalars in $(0,1)$ as $\sigma(\cdot)$
of an unconstrained raw weight, and positive scalars as $\softplus(\cdot)$.
TMMFormer's velocity-reinjection gain is initialized so that
$\softplus(\nu^{\mathrm{raw}})\approx 1$; training therefore begins in the
$\nu\equiv 1$ YuriiFormer regime and learns where to deviate. MuonFormer
reshapes each token's $d$-dimensional update into an $(H,D)=(12,64)$ matrix
(heads $\times$ head dimension) and applies the quintic Newton--Schulz
iteration independently per token for $K=5$ steps to recover the orthogonal
polar factor, then reshapes back to $d$.

\subsection{Two-Optimizer Parameter Training}
\label{sec:appendix-two-opt}

Network parameters are split into four groups, each with its own optimizer
and learning rate (Table~\ref{tab:hp-optimizers}):

\begin{itemize}
  \item \textbf{$2$D weight matrices} (attention/MLP projections):
        \textbf{Muon}, Nesterov momentum $0.95$, weight decay $0$, peak
        learning rate $0.02$ (TS) / $0.004$ (OWT).
  \item \textbf{Token/position embeddings}: AdamW, lr $6\times10^{-4}$,
        weight decay $0.1$.
  \item \textbf{LayerNorm gains}: AdamW, lr $6\times10^{-4}$, weight
        decay $0$.
  \item \textbf{Learned per-layer scalars} $\omega_\ell$ (the
        $\mathrm{raw}$ parameters): AdamW, lr $3\times10^{-3}$, weight
        decay $0$.
\end{itemize}

The Muon learning rate is the only optimization hyperparameter that differs
between corpora; everything else is identical for TS and OWT. The single
global schedule multiplier (linear warmup, then cosine decay to
$0.1\times$ peak) is applied to all four groups simultaneously, so the
ratio between the four learning rates is held fixed throughout training.

\subsection{Hyperparameter Table}
\label{sec:appendix-hparam-table}

\begin{tablebox}{Backbone (shared by all variants)}
\setlength{\tabcolsep}{4pt}
\renewcommand{\arraystretch}{1.15}
\begin{tabular}{@{}ll@{}}
\textbf{Hyperparameter} & \textbf{Value} \\
\midrule
Layers / heads / $d$            & $12$ / $12$ / $768$ \\
Context length                  & $1024$ \\
Tokenizer / $|\mathcal V|$      & GPT-2 BPE / $50{,}304$ \\
Output head                     & weight-tied \\
Params (vanilla / aux-stream)   & $124$M / $\approx 163$M \\
\end{tabular}
\end{tablebox}
\captionof{table}{Backbone configuration, identical for all variants and
both corpora.}
\label{tab:hp-backbone}

\begin{tablebox}{Optimization budget}
\setlength{\tabcolsep}{4pt}
\renewcommand{\arraystretch}{1.15}
\begin{tabular}{@{}lll@{}}
\textbf{Hyperparameter} & \textbf{TS} & \textbf{OWT} \\
\midrule
Total steps                     & $10{,}000$ & $30{,}000$ \\
Warmup steps                    & $1{,}000$  & $3{,}000$ \\
LR schedule                     & \multicolumn{2}{c}{warmup $\to$ cosine} \\
Min-LR ratio                    & \multicolumn{2}{c}{$0.1$} \\
Micro-batch                     & \multicolumn{2}{c}{$8$} \\
Grad accumulation               & \multicolumn{2}{c}{$60$} \\
Effective batch (seqs)          & \multicolumn{2}{c}{$480$} \\
Tokens / step                   & \multicolumn{2}{c}{$\approx 4.9\times10^{5}$} \\
\end{tabular}
\end{tablebox}
\captionof{table}{Optimization budget. Only total and warmup steps differ
between TS and OWT.}
\label{tab:hp-budget}

\begin{tablebox}{Optimizers (per parameter group)}
\setlength{\tabcolsep}{4pt}
\renewcommand{\arraystretch}{1.15}
\begin{tabular}{@{}lll@{}}
\textbf{Hyperparameter} & \textbf{TS} & \textbf{OWT} \\
\midrule
Muon ($2$D weights), lr         & $0.02$ & $0.004$ \\
\quad momentum / Nesterov / wd  & \multicolumn{2}{c}{$0.95$ / yes / $0$} \\
AdamW (embeddings), lr / wd     & \multicolumn{2}{c}{$6\!\times\!10^{-4}$ / $0.1$} \\
AdamW (LayerNorm), lr / wd      & \multicolumn{2}{c}{$6\!\times\!10^{-4}$ / $0$} \\
AdamW (scalars), lr / wd        & \multicolumn{2}{c}{$3\!\times\!10^{-3}$ / $0$} \\
\end{tabular}
\end{tablebox}
\captionof{table}{Optimizer routing and hyperparameters per parameter group.
The Muon peak LR is the only TS/OWT difference.}
\label{tab:hp-optimizers}

\begin{tablebox}{System and evaluation}
\setlength{\tabcolsep}{4pt}
\renewcommand{\arraystretch}{1.15}
\begin{tabular}{@{}ll@{}}
\textbf{Hyperparameter} & \textbf{Value} \\
\midrule
Precision                       & \texttt{bfloat16} autocast \\
Compilation                     & \texttt{torch.compile} \\
Data parallel                   & $2$-GPU DDP \\
Seeds                           & single ($42$) \\
Val interval / \#batches        & $100$ steps / $160$ \\
Checkpoint                      & best val CE \\
Downstream harness              & \texttt{lm-eval} v$0.4.3$ \\
HellaSwag / ARC-Easy            & $10$-shot / $25$-shot \\
Downstream metric               & \texttt{acc\_norm} \\
\end{tabular}
\end{tablebox}
\captionof{table}{System and evaluation settings, identical for all variants
and both corpora.}
\label{tab:hp-system}

\subsection{Pretraining and Downstream Evaluation}
\label{sec:appendix-eval}

During training we evaluate the validation cross-entropy every $100$ steps
on $160$ fixed batches and keep the checkpoint with the lowest value
(\texttt{best.pt}); pretraining quality in the main text is this best
validation loss in nats/token. Downstream transfer is measured only on the
best \emph{OWT} checkpoint of each variant---TS-pretrained checkpoints sit
at chance on these benchmarks because the TinyStories distribution is too
narrow to transfer---using \texttt{lm-evaluation-harness}~v0.4.3 with
HellaSwag ($10$-shot) and ARC-Easy ($25$-shot). We report length-normalized
accuracy (\texttt{acc\_norm}), the harness default for these
multiple-choice tasks. All runs use a single seed ($42$); differences
within ${\sim}0.01$ \texttt{acc\_norm} on these benchmarks at this model
scale are within seed noise and are not interpreted as signal.

\subsection{Wall-Clock}
\label{sec:appendix-wall-clock}

Per-step wall-clock for the OpenWebText runs, as printed by the training loop
(hardware varies across SLURM jobs, so values are indicative). The vanilla
transformer and the momentum-stream variants are all comparable:
VanillaTransformer ${\approx}3.3$--$3.4$\,s/step, TMMFormer
${\approx}2.7$--$3.9$\,s/step, YuriiFormer ${\approx}2.8$\,s/step, and
AdamFormer ${\approx}3.4$\,s/step. The TMMFormer auxiliary velocity stream and
per-layer scalars add negligible cost relative to attention and the MLP, so
TMMFormer does not materially increase runtime over the vanilla baseline. The
only variant with a large penalty is MuonFormer, whose per-token
Newton--Schulz iteration runs at ${\approx}10$--$16$\,s/step
(${\approx}3$--$5\times$ slower).

\subsection{Detailed Results and Ablation}
\label{sec:detailed-results}

\begin{tablebox}{Full results}
\setlength{\tabcolsep}{4pt}
\renewcommand{\arraystretch}{1.15}
\begin{tabular}{@{}lcccc@{}}
& \multicolumn{2}{c}{\textbf{val loss} $\downarrow$}
& \multicolumn{2}{c}{\textbf{acc\_norm} (\%) $\uparrow$} \\
\cmidrule(lr){2-3}\cmidrule(lr){4-5}
\textbf{Variant} & TS & OWT & HS & ARC \\
\midrule
VanillaTransformer & $1.1569$ & $3.0078$ & $30.20$ & $41.67$ \\
AdamFormer         & $1.1528$ & $2.9911$ & $30.96$ & $43.39$ \\
AdamWFormer        & $1.1472$ & $2.9883$ & $30.08$ & $41.88$ \\
\textbf{TMMFormer} & $\mathbf{1.1284}$ & $\mathbf{2.9342}$
                   & $\mathbf{31.82}$  & $\mathbf{43.43}$ \\
\end{tabular}
\end{tablebox}
\captionof{table}{Full results: best val loss (nats/token) and OWT downstream
\texttt{acc\_norm} (\%); best per column bold. Only variants with results on
every task are listed; partial-task variants (MuonFormer, SOAPFormer) are
reported in the text below.}
\label{tab:full-results}

\begin{figure}[t]
\centering
\includegraphics[width=\textwidth]{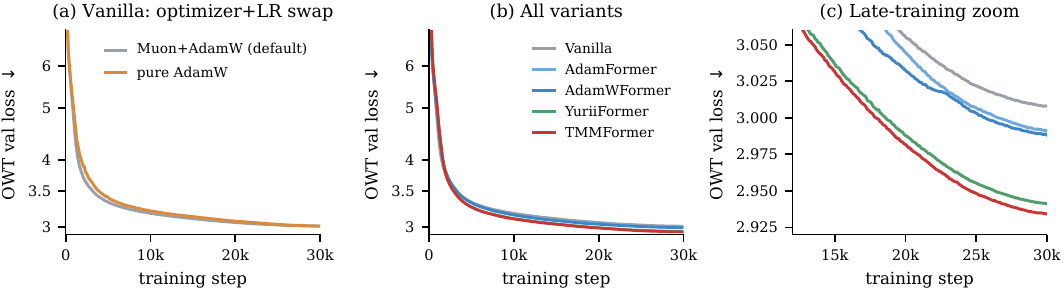}
\caption{OWT validation-loss training curves (real per-step values from
the canonical SLURM runs). \textbf{(a)} Vanilla under the default
Muon$+$AdamW recipe vs.\ a single pure-AdamW optimizer (optimizer$+$LR
swap). \textbf{(b)} All optimizer-inspired variants. \textbf{(c)} The same,
zoomed to late training, where the ordering TMMFormer $<$ YuriiFormer $<$
AdamFormer ${\approx}$ AdamWFormer $<$ Vanilla is clear.}
\label{fig:loss-curves}
\end{figure}

\paragraph{Full results.}
Table~\ref{tab:full-results} and Figure~\ref{fig:results-overview} report
all variants against the vanilla baseline of
Section~\ref{sec:preliminary}; per-step OWT training curves for every
variant are in Figure~\ref{fig:loss-curves}.
TMMFormer attains the lowest validation loss on both corpora ($1.1284$ on
TS, $2.9342$ on OWT) and the best downstream transfer ($31.8\%$ HellaSwag,
$43.4\%$ ARC-Easy); the momentum-stream design improves on vanilla by
${\approx}0.029$ nats on TS and ${\approx}0.074$ on OWT, and the
pretraining ordering (TMMFormer $>$ AdamFormer ${\approx}$ AdamWFormer $>$
Vanilla) transfers \emph{exactly} to downstream accuracy, so the gain is
genuine generalization rather than a pretraining-loss artifact. AdamFormer
and AdamWFormer recover part of the gap---an adaptive per-coordinate update
is better than none---but the diagonal $\sqrt{S}$ preconditioner is largely
absorbed by the subsequent $\LN_u$
(Appendix~\ref{sec:layernorm-redundancy}), so they fall well short of the
second-order momentum recurrence. MuonFormer and SOAPFormer do not have
results on every task and are therefore omitted from
Table~\ref{tab:full-results}; we report them here. MuonFormer converges on
both corpora but its spectral preconditioning yields no gain ($1.1503$ on TS,
between the two Adam variants, and $3.0096$ on OWT, no better than vanilla).
SOAPFormer converges on TinyStories ($1.1431$, the weakest matrix-%
preconditioned variant) but its per-token full-matrix $R^{-1/2}$ is
poorly conditioned in the per-token covariance estimate and did not complete the
OpenWebText run.
These two negative results reinforce the central finding:
among updates that \emph{do} train stably inside the residual stream, the
Nesterov-style and triple-momentum second-order recurrences are the most
effective Transformer blocks, consistent with the optimization view of
Section~\ref{sec:preliminary} (triple momentum has an optimal rate on the
strongly-convex quadratic model used to study the depth recurrence) and
analyzed architecturally in
Appendix~\ref{sec:momentum-residual-filtering}.

\begin{tablebox}{Optimizer-on-2D ablation (Vanilla, OWT, step-aligned)}
\setlength{\tabcolsep}{6pt}
\renewcommand{\arraystretch}{1.15}
\begin{tabular}{@{}rccc@{}}
\textbf{step} & \textbf{pure AdamW} & \textbf{Muon$+$AdamW}
              & \textbf{$\Delta$} \\
\midrule
$1$k  & $4.844$ & $4.638$ & $+0.206$ \\
$2$k  & $3.923$ & $3.734$ & $+0.189$ \\
$3$k  & $3.651$ & $3.528$ & $+0.122$ \\
$5$k  & $3.397$ & $3.334$ & $+0.063$ \\
$7$k  & $3.294$ & $3.251$ & $+0.043$ \\
$10$k & $3.213$ & $3.178$ & $+0.035$ \\
$15$k & $3.132$ & $3.107$ & $+0.025$ \\
$20$k & $3.074$ & $3.055$ & $+0.018$ \\
$25$k & $3.030$ & $3.022$ & $+0.009$ \\
$30$k & $3.011$ & $3.008$ & $+0.002$ \\
\end{tabular}
\end{tablebox}
\captionof{table}{Step-aligned OWT val loss for \emph{Vanilla} under the two
parameter-training optimizers (Muon hybrid vs.\ pure AdamW). The
Muon-hybrid early-training advantage decays monotonically to within one
val-eval by step $30$k.}
\label{tab:opt-ablation}

\paragraph{Optimizer and learning-rate ablation.}
To check that the TMMFormer advantage is architectural and not an artifact
of the Muon$+$AdamW training recipe, we run two independent OWT ablations.

\emph{Optimizer on $2$D weights.} We replace the Muon optimizer that
updates the $48$ $2$D matrix weights (qkv, out, w1, w2 across the $12$
blocks) with AdamW, holding the schedule, batch, and the embedding and
LayerNorm optimizer fixed at AdamW $@\,6{\times}10^{-4}$. Combined with the
two architectures this is a full $2{\times}2$: Vanilla goes from $3.0078$
(Muon hybrid) to $3.0103$ (pure AdamW), $\Delta{=}0.0025$, while TMMFormer
goes from $2.9342$ to $2.9696$, $\Delta{=}0.0354$. TMMFormer beats Vanilla
under both optimizers (architectural gap $0.074$ under the Muon hybrid,
$0.041$ under pure AdamW). The two optimizers are essentially
indistinguishable on Vanilla at the final step (within one val-eval) but
distinguishable on TMMFormer ($\approx 14{\times}$ larger gap); the
step-aligned Vanilla curves (Table~\ref{tab:opt-ablation},
Figure~\ref{fig:loss-curves}a) show that the Muon-hybrid early-training
lead on Vanilla decays monotonically and disappears by $30$k, so the
converged Vanilla numbers are essentially on top of each other. The
architectural ordering survives in either column.

\emph{Learning-rate sweep (partial).} Holding the Muon$+$AdamW recipe
fixed, we halve the Muon peak learning rate from $4{\times}10^{-3}$ to
$2{\times}10^{-3}$. Vanilla degrades from $3.0078$ to $3.0288$ ($+0.021$
nats) and TMMFormer from $2.9342$ to $2.9634$ ($+0.029$); the architectural
gap shrinks from $0.074$ to $0.066$ (${\approx}11\%$ compression). The
sweep is partial: only the $\times 0.5$ Muon-LR cell ran for each
architecture; the $\times 2$ Muon-LR cells and the four pure-AdamW LR
cells of the originally designed $2{\times}2{\times}2$ grid were submitted
and then cancelled, and seed count is $N{=}1$. The robustness claim is
therefore that the TMMFormer$>$Vanilla ordering is preserved under halving
the Muon LR, not across a full LR grid.

In both ablations the architectural effect (TMMFormer vs.\ Vanilla,
${\approx}0.07$ nats on OWT) is an order of magnitude larger than the
optimizer- or LR-induced effect on either architecture, and consistent in
sign across both perturbations.

\paragraph{Parameter-matched controls.}
TMMFormer carries ${\approx}39$M more parameters than the vanilla
$12$L/$768$d backbone ($163.8$M vs.\ $124.4$M); over $99\%$ of this is the
duplicate token$+$position embedding table that initializes the velocity
stream, and the within-block additions (two extra LayerNorms per block) are
negligible in both parameter count and per-step FLOPs. To rule out a pure
parameter-count effect we train a Vanilla variant whose parameter count
matches TMMFormer's, with the optimization recipe of
Section~\ref{sec:appendix-two-opt} held fixed (TinyStories, $10$k steps,
seed $42$; Muon on $2$D weights with peak lr $2{\times}10^{-2}$ and AdamW
on embeddings and LayerNorms with lr $6{\times}10^{-4}$; warmup $1$k steps
followed by cosine decay to a floor of $0.1{\times}$ peak; gradient clip
$1.0$; effective batch of $480$ sequences of length $1024$). The match is
single-axis: \emph{width-$900$} ($12$L, $d_{\mathrm{model}}{=}900$, $12$
heads, $d_{\mathrm{head}}{=}75$, $162.86$M---$0.6\%$ below TMMFormer) is
the clean control. A secondary \emph{depth-$18$} run ($18$L,
$d_{\mathrm{model}}{=}768$, $166.85$M) was disrupted by partition-walltime
requeues and is reported only as suggestive.

Width-$900$ Vanilla reaches a best validation loss of $1.1454$ on
TinyStories, versus $1.1578$ for the default Vanilla and $1.1272$ for
TMMFormer, closing
$(1.1578-1.1454)/(1.1578-1.1272)\approx 40\%$ of the TMMFormer--Vanilla
gap. The remaining ${\approx}0.0182$ nats is roughly $9\times$ the
per-seed standard deviation reported in the seed-variance arm
($\le 0.0028$ at the default sizes), so it is well outside seed noise.
The depth-$18$ Vanilla appears to close a larger fraction of the gap
(best validation ${\approx}1.130$); pending a clean rerun we do not draw
any conclusion from the width--depth contrast. The headline conclusion
stands either way: the bulk of TMMFormer's advantage on TinyStories is
\emph{not} explained by parameter count---giving Vanilla TMMFormer's
parameter budget, spent as extra width, recovers under half the gap. A
full isolation of velocity \emph{dynamics} from the duplicate velocity
\emph{embeddings} would require pairing this control with a TMM variant
that drops the embedding table but keeps the dynamics (and vice versa);
we leave that to future work.

%% file: appendix/precond_redundancy.tex
\section{Preconditioning Redundancy in Pre-Norm Transformers}
\label{sec:preconditioning-redundancy}
\label{sec:layernorm-redundancy}

Adam- or RMSProp-style diagonal preconditioning is most useful when different
coordinates of an update have substantially different scales. In a pre-norm
Transformer, however, each attention or MLP oracle receives normalized token
representations. This does not make every diagonal preconditioner algebraically
redundant: in general, $\LN(Dx) \neq \LN(x)$ for a non-scalar positive diagonal
matrix $D$. Rather, LayerNorm motivates a weaker and more useful condition: the
oracle outputs may already have nearly balanced coordinate-wise second moments.
Under this condition, the Adam- or RMSProp-style token-space preconditioner
collapses to an approximate scalar step-size rescaling.

For $x \in \mathbb{R}^d$, define the zero-mean LayerNorm map without learned
gain or bias by
\begin{equation*}
    \LN(x)
    =
    \frac{x-\bar{x}\mathbf{1}}
         {\|x-\bar{x}\mathbf{1}\|_2/\sqrt{d}},
    \quad
    \bar{x}
    =
    \frac{1}{d}\mathbf{1}^{\top}x .
\end{equation*}

\begin{tcolorbox}[thmblue, title={\textbf{Balanced moments flatten gains}}]
\begin{theorem}
\label{thm:layernorm-diagonal-redundancy}
Consider an Adam- or RMSProp-style diagonal preconditioner
\begin{equation*}
    D_s
    =
    \diag\!\left(
    \frac{1}{\sqrt{s_i}+\delta}
    \right)_{i=1}^d,
\end{equation*}
with $\delta\ge0$, applied to an update direction before the update LayerNorm
$\LN_u$ used in AdamFormer and RMSPropFormer. Let
\begin{align*}
    \rho_i
    &=
    \sqrt{s_i},
    &
    \rho_{\min}
    &=
    \min_i \rho_i,
    &
    \rho_{\max}
    &=
    \max_i \rho_i .
\end{align*}
Assume the second-moment stream is nearly coordinate-balanced:
\begin{equation*}
    \frac{\rho_{\max}^2}{\rho_{\min}^2}
    \le 1+\epsilon
    \quad
    \text{for some } \epsilon \in [0,1],
    \quad
    \rho_{\min}>0.
\end{equation*}
Then there exists a scalar $\alpha>0$ such that
\begin{align*}
    \left\|\frac{D_s}{\alpha}-I\right\|_2
    &\le
    \frac{\sqrt{1+\epsilon}-1}{1+\delta/\rho_{\max}}
    \le
    \frac{\epsilon}{2}.
\end{align*}
Consequently, the diagonal Adam- or RMSProp-style preconditioner in this
Transformer substep is approximately a scalar multiple of the identity. Since
$\LN_u$ is invariant to positive scalar rescaling, the scalar part of the
preconditioner is absorbed by the update LayerNorm. Thus any nontrivial
contribution of diagonal preconditioning must come from the $O(\epsilon)$
non-scalar deviation of $D_s$ from a scalar matrix.
\end{theorem}
\end{tcolorbox}

\begin{proof}
The diagonal entries of $D_s$ are
\begin{equation*}
    d_i = \frac{1}{\rho_i+\delta}.
\end{equation*}
Since $d_i$ is decreasing in $\rho_i$, the largest and smallest diagonal
entries are
\begin{align*}
    d_{\max}
    &=
    \frac{1}{\rho_{\min}+\delta},
    &
    d_{\min}
    &=
    \frac{1}{\rho_{\max}+\delta}.
\end{align*}
Choose $\alpha=d_{\max}$. Then every normalized diagonal entry satisfies
\begin{align*}
    \frac{d_i}{\alpha}
    &=
    \frac{\rho_{\min}+\delta}{\rho_i+\delta}
    \\
    &\in
    \left[
        \frac{\rho_{\min}+\delta}
             {\rho_{\max}+\delta},
        1
    \right].
\end{align*}
Because $D_s/\alpha-I$ is diagonal,
\begin{align*}
    \left\|\frac{D_s}{\alpha}-I\right\|_2
    &=
    1-
    \frac{\rho_{\min}+\delta}
         {\rho_{\max}+\delta}
=
    \frac{\rho_{\max}-\rho_{\min}}
         {\rho_{\max}+\delta}.
\end{align*}
The balance assumption gives
\begin{equation*}
    \rho_{\max}
    \le
    \sqrt{1+\epsilon}\,\rho_{\min},
\end{equation*}
or equivalently
\begin{equation*}
    \rho_{\max}-\rho_{\min}
    \le
    \rho_{\max}
    \left(
        1-\frac{1}{\sqrt{1+\epsilon}}
    \right).
\end{equation*}
Therefore
\begin{align*}
    \left\|\frac{D_s}{\alpha}-I\right\|_2
    &\le
    \frac{
        1-\frac{1}{\sqrt{1+\epsilon}}
    }{
        1+\delta/\rho_{\max}
    }
    \\
    &\le
    1-\frac{1}{\sqrt{1+\epsilon}}
    \\
    &\le
    \sqrt{1+\epsilon}-1
    \\
    &\le
    \frac{\epsilon}{2},
\end{align*}
where the final inequality holds for $\epsilon \in [0,1]$ by concavity of the
square-root function. This proves that $D_s$ differs from a scalar multiple
of the identity by at most $O(\epsilon)$ in spectral norm.

Finally, both AdamFormer and RMSPropFormer apply $\LN_u$ after the diagonal
preconditioner. For any positive scalar $c$, $\LN_u(cz)=\LN_u(z)$ up to the
fixed numerical epsilon in the LayerNorm denominator. Thus the scalar component
of $D_s$ cannot provide an independent update direction after $\LN_u$.
\end{proof}

\begin{remark}[Interpretation]
\label{rem:layernorm-role}
The balance assumption is natural in this setting because the Adam- or RMSProp-style
second-moment stream is driven by oracle outputs
$q=\mathcal{O}_{\ell}(\LN(x))$, rather than raw, unnormalized token states.
The oracle input is normalized, the coordinate projections are dense and
learned, and the final adaptive update is again normalized by $\LN_u$ before
entering the residual stream. These architectural features make large
persistent coordinate-scale disparities less central than in parameter-space
optimization, where Adam is most useful.

The theorem shows that, under this balanced-second-moment condition, the
Adam- or RMSProp-style diagonal preconditioner in token space is nearly a scalar
matrix. Since $\LN_u$ removes positive scalar rescalings of the update, the
scalar part of the preconditioner does not create a new residual direction. Any
useful effect must therefore come from the small non-scalar deviation of
$D_s$ from a scalar multiple of the identity.
\end{remark}

%% file: appendix/momentum_theory.tex
\section{Momentum as Second-Order Residual-Stream Filtering}
\label{sec:momentum-residual-filtering}

We now give a local explanation for why momentum-stream Transformer variants
can outperform a vanilla pre-norm Transformer at matched depth. The point is
architectural: a momentum stream changes the forward residual dynamics from a
first-order recurrence to a second-order recurrence. In a linearized local
model, this gives a richer polynomial filter over token-feature modes.

\subsection{Local Linearized Sandbox}
\label{sec:linearized-sandbox}

Let the hidden state at layer $\ell$ be
\begin{equation*}
    X_\ell
    \in
    \mathbb{R}^{T\times d},
\end{equation*}
and let $X^\star \in \mathbb{R}^{T\times d}$ be a task-relevant hidden
representation. We assume a local quadratic surrogate energy
\begin{equation*}
    \mathcal{F}(X)
    =
    \frac{1}{2}
    \langle X-X^\star,\,
    H(X-X^\star)\rangle,
\end{equation*}
where $H$ is self-adjoint positive definite on token-embedding space, with
\begin{equation*}
    0<\mu
    \le
    \lambda_i(H)
    \le
    L,
    \qquad
    \kappa=\frac{L}{\mu}.
\end{equation*}
In the clean sandbox case, the combined Transformer oracle is aligned with the
negative gradient of this surrogate:
\begin{equation}
    G(X)
    =
    -H(X-X^\star).
    \label{eq:linearized-oracle}
\end{equation}
The language-modeling logits are produced by the shared output head
\begin{equation*}
    Z(X)=XW^\top,
\end{equation*}
where $W\in\mathbb{R}^{|\mathcal{V}|\times d}$ is the token embedding/output
matrix and $|\mathcal{V}|$ is the vocabulary size. Thus
$Z(X)\in\mathbb{R}^{T\times |\mathcal{V}|}$.
and the language-modeling loss is
\begin{equation*}
    \mathcal{L}_{\mathrm{LM}}(X)
    =
    \operatorname{CE}(Z(X),Y).
\end{equation*}
Assume that, near $X^\star$, this loss is $C$-smooth as a function of the
hidden representation and that $X^\star$ is a local minimizer. Then
\begin{equation}
    \mathcal{L}_{\mathrm{LM}}(X)
    -
    \mathcal{L}_{\mathrm{LM}}(X^\star)
    \le
    \frac{C}{2}
    \|X-X^\star\|_F^2 .
    \label{eq:lm-smoothness}
\end{equation}

\subsection{Vanilla as a First-Order Filter}
\label{sec:vanilla-filter}

Let $E_\ell=X_\ell-X^\star$. In the sandbox, a vanilla residual step is
\begin{equation*}
    X_{\ell+1}
    =
    X_\ell
    -
    \eta H(X_\ell-X^\star),
\end{equation*}
so
\begin{equation*}
    E_{\ell+1}
    =
    (I-\eta H)E_\ell .
\end{equation*}
If $H=Q\Lambda Q^\top$, each eigenmode evolves independently:
\begin{equation*}
    e_{\ell+1,i}
    =
    (1-\eta\lambda_i)e_{\ell,i}.
\end{equation*}
After $N$ layers,
\begin{equation*}
    E_N
    =
    p_N(H)E_0,
    \qquad
    p_N(\lambda)
    =
    (1-\eta\lambda)^N .
\end{equation*}
Thus vanilla implements a first-order polynomial filter over the local
token-feature spectrum.

\begin{tcolorbox}[thmmag, title={\textbf{Best uniform vanilla contraction}}]
\begin{lemma}
\label{lem:best-vanilla-contraction}
For $\lambda\in[\mu,L]$, the best fixed scalar step size for the vanilla update
is
\begin{equation*}
    \eta^\star
    =
    \frac{2}{L+\mu}.
\end{equation*}
The resulting worst-case contraction factor is
\begin{equation*}
    \rho_{\mathrm{vanilla}}
    =
    \max_{\lambda\in[\mu,L]}
    |1-\eta^\star\lambda|
    =
    \frac{\kappa-1}{\kappa+1}.
\end{equation*}
Consequently,
\begin{equation*}
    \|E_N^{\mathrm{vanilla}}\|_F
    \le
    \rho_{\mathrm{vanilla}}^N
    \|E_0\|_F .
\end{equation*}
\end{lemma}
\end{tcolorbox}

\begin{proof}
For fixed $\eta$, the worst-case contraction over $[\mu,L]$ is
\begin{equation*}
    \max\{
    |1-\eta\mu|,
    |1-\eta L|
    \}.
\end{equation*}
The optimum equalizes the endpoint magnitudes:
\begin{equation*}
    1-\eta\mu
    =
    -(1-\eta L),
\end{equation*}
which gives $\eta^\star=2/(L+\mu)$. Substitution yields
\begin{equation*}
    1-\eta^\star\mu
    =
    \frac{L-\mu}{L+\mu}
    =
    \frac{\kappa-1}{\kappa+1}.
\end{equation*}
Since $H$ is self-adjoint positive definite, the operator norm of
$I-\eta^\star H$ is the maximum absolute eigenvalue over the interval.
Iterating gives the result.
\end{proof}

\subsection{Momentum as a Second-Order Filter}
\label{sec:momentum-filter}

Consider a linearized single-oracle momentum substep:
\begin{align*}
    \widetilde{X}_\ell
    &=
    X_\ell + aV_\ell,
    \\
    V_{\ell+1}
    &=
    bV_\ell
    -
    \eta H(\widetilde{X}_\ell-X^\star),
    \\
    X_{\ell+1}
    &=
    X_\ell + cV_{\ell+1}.
\end{align*}
This family includes heavy-ball updates $(a=0)$, the YuriiFormer
Nesterov-style lookahead update from prior work $(a>0,c=1)$
\citep{zimin2026yuriiformer}, and TMM-style updates $(a>0)$ with learned $c$.
Because $V_\ell$ stores information from previous residual updates, each
eigenmode follows a second-order recurrence. Indeed, for an eigenmode with
eigenvalue $\lambda$, write the scalar error and velocity as $e_\ell$ and
$v_\ell$. Then
\begin{align*}
    v_{\ell+1}
    &=
    b v_\ell
    -
    \eta\lambda(e_\ell+a v_\ell)
    \\
    &=
    -\eta\lambda e_\ell
    +
    (b-a\eta\lambda)v_\ell,
    \\
    e_{\ell+1}
    &=
    e_\ell + c v_{\ell+1}.
\end{align*}
When $c>0$, $v_\ell=(e_\ell-e_{\ell-1})/c$, and therefore
\begin{align*}
    e_{\ell+1}
    &=
    (1-c\eta\lambda)e_\ell
    +
    (b-a\eta\lambda)
    (e_\ell-e_{\ell-1})
    \\
    &=
    \alpha(\lambda)e_\ell
    -
    \theta(\lambda)e_{\ell-1},
\end{align*}
where
\begin{align*}
    \alpha(\lambda)
    &=
    1+b-(a+c)\eta\lambda,
    \\
    \theta(\lambda)
    &=
    b-a\eta\lambda .
\end{align*}
Thus, in general,
\begin{equation*}
    e_{\ell+1,i}
    =
    \alpha(\lambda_i)e_{\ell,i}
    -
    \theta(\lambda_i)e_{\ell-1,i},
\end{equation*}
for coefficients determined by $(a,b,c,\eta)$. Therefore
\begin{equation*}
    E_N
    =
    q_N(H)E_0,
\end{equation*}
where $q_N$ is generated by a second-order recurrence. This is a richer filter
family than the vanilla filter $(1-\eta\lambda)^N$.

\begin{tcolorbox}[thmmag, title={\textbf{Momentum improves contraction}}]
\begin{lemma}
\label{lem:momentum-contraction}
If $\kappa>1$, there exist stable momentum coefficients such that
\begin{equation*}
    \rho_{\mathrm{mom}}
    =
    \frac{\sqrt{\kappa}-1}
         {\sqrt{\kappa}+1}
    <
    \frac{\kappa-1}{\kappa+1}
    =
    \rho_{\mathrm{vanilla}} .
\end{equation*}
Thus a second-order momentum recurrence can have a strictly better finite-depth
worst-case contraction factor than the vanilla recurrence.
\end{lemma}
\end{tcolorbox}

\begin{proof}
It suffices to exhibit one stable momentum choice. For the quadratic energy
$\mathcal{F}(x)=\frac{1}{2}x^\top Hx$, choose the classical heavy-ball
parameters
\begin{align*}
    \eta_{\mathrm{HB}}
    &=
    \frac{4}{(\sqrt{L}+\sqrt{\mu})^2},
    \\
    \beta_{\mathrm{HB}}
    &=
    \left(
    \frac{\sqrt{L}-\sqrt{\mu}}
         {\sqrt{L}+\sqrt{\mu}}
    \right)^2 .
\end{align*}
For each eigenmode $\lambda\in[\mu,L]$, the recurrence is
\begin{equation*}
    e_{\ell+1}
    =
    (1-\eta_{\mathrm{HB}}\lambda+\beta_{\mathrm{HB}})
    e_\ell
    -
    \beta_{\mathrm{HB}}e_{\ell-1}.
\end{equation*}
The characteristic polynomial is
\begin{equation*}
    r^2
    -
    (1-\eta_{\mathrm{HB}}\lambda+\beta_{\mathrm{HB}})r
    +
    \beta_{\mathrm{HB}}
    =
    0 .
\end{equation*}
Classical Chebyshev semi-iterative analysis gives a mode-wise bound with rate
\begin{equation*}
    \sqrt{\beta_{\mathrm{HB}}}
    =
    \frac{\sqrt{L}-\sqrt{\mu}}
         {\sqrt{L}+\sqrt{\mu}}
    =
    \frac{\sqrt{\kappa}-1}
         {\sqrt{\kappa}+1}.
\end{equation*}
Thus, for a constant $C_{\mathrm{mom}}$ depending on the initial velocity,
\begin{equation*}
    \|E_N^{\mathrm{mom}}\|_F
    \le
    C_{\mathrm{mom}}
    \rho_{\mathrm{mom}}^N
    \|E_0\|_F .
\end{equation*}
It remains to compare the factors. Let $s=\sqrt{\kappa}>1$. Then
\begin{align*}
    \rho_{\mathrm{vanilla}}
    -
    \rho_{\mathrm{mom}}
    &=
    \frac{s^2-1}{s^2+1}
    -
    \frac{s-1}{s+1}
    \\
    &=
    \frac{2s(s-1)}
         {(s^2+1)(s+1)}
    >
    0 .
\end{align*}
Hence $\rho_{\mathrm{mom}}<\rho_{\mathrm{vanilla}}$.
\end{proof}

\begin{corollary}[TMM contains the YuriiFormer update in the linearized class]
\label{cor:tmm-contains-yurii}
In the linearized eigenmode model, the polynomial family realizable by TMM
contains the polynomial family realizable by the YuriiFormer update. Therefore
\begin{equation*}
    \inf_{q\in\mathcal{Q}_N^{\mathrm{TMM}}}
    \max_{\lambda\in[\mu,L]} |q(\lambda)|
    \le
    \inf_{q\in\mathcal{Q}_N^{\mathrm{Yurii}}}
    \max_{\lambda\in[\mu,L]} |q(\lambda)|.
\end{equation*}
\end{corollary}

\begin{proof}
YuriiFormer fixes the residual reinjection coefficient to $c=1$, while
TMMFormer learns $c=\nu_\ell$. Setting $\nu_\ell=1$ in TMMFormer recovers the
YuriiFormer update. Thus every linearized polynomial filter attainable by
YuriiFormer is also attainable by TMMFormer. Taking the infimum over a larger
class cannot increase the worst-case spectral error.
\end{proof}

\subsection{From Representation Error to Local LM Loss}
\label{sec:representation-to-lm-loss}

In this subsection, $N$ counts full Transformer layers. Each layer is modeled by
its effective linearized residual update, abstracting over the internal
attention--MLP Lie--Trotter substeps.

\begin{tcolorbox}[thmblue, title={\textbf{Momentum lowers the local loss bound}}]
\begin{theorem}
\label{thm:momentum-lower-lm-loss-bound}
Under the local linearized assumptions above, let $X_N^V$ be the representation
produced by a vanilla Transformer after $N$ layers. Let $X_N^M$ be the
representation produced by a momentum Transformer after $N$ layers, where one
layer is represented by the corresponding effective linearized residual update.
If both models start from the same $X_0$, then
\begin{align*}
    &\mathcal{L}_{\mathrm{LM}}(X_N^V)
    -
    \mathcal{L}_{\mathrm{LM}}(X^\star)
    \le
    \frac{C}{2}
    \rho_{\mathrm{vanilla}}^{2N}
    \|X_0-X^\star\|_F^2,
    \\
    &\mathcal{L}_{\mathrm{LM}}(X_N^M)
    -
    \mathcal{L}_{\mathrm{LM}}(X^\star)\\
    &\le
    \frac{C}{2}
    C_{\mathrm{mom}}^2
    \rho_{\mathrm{mom}}^{2N}
    \|X_0-X^\star\|_F^2 .
\end{align*}
Since $\rho_{\mathrm{mom}}<\rho_{\mathrm{vanilla}}$, for any fixed
$C_{\mathrm{mom}}<\infty$ there exists
\begin{equation*}
    N_0
    =
    \max\left\{
    0,\,
    \left\lceil
    \frac{\log C_{\mathrm{mom}}}
         {\log(\rho_{\mathrm{vanilla}}/\rho_{\mathrm{mom}})}
    \right\rceil
    \right\}
\end{equation*}
such that for all $N\ge N_0$, the momentum upper bound is strictly lower than
the vanilla upper bound.
\end{theorem}
\end{tcolorbox}

\begin{proof}
By the smoothness assumption in~\eqref{eq:lm-smoothness},
\begin{equation*}
    \mathcal{L}_{\mathrm{LM}}(X)
    -
    \mathcal{L}_{\mathrm{LM}}(X^\star)
    \le
    \frac{C}{2}
    \|X-X^\star\|_F^2 .
\end{equation*}
Because $N$ counts full layers, applying the vanilla layer contraction from
Lemma~\ref{lem:best-vanilla-contraction} for $N$ iterations gives
\begin{equation*}
    \|X_N^V-X^\star\|_F
    \le
    \rho_{\mathrm{vanilla}}^N
    \|X_0-X^\star\|_F .
\end{equation*}
Substituting this representation-error bound into
\eqref{eq:lm-smoothness} yields the vanilla LM-loss bound.

Similarly, applying the effective momentum layer contraction from
Lemma~\ref{lem:momentum-contraction} for $N$ layers gives
\begin{equation*}
    \|X_N^M-X^\star\|_F
    \le
    C_{\mathrm{mom}}
    \rho_{\mathrm{mom}}^N
    \|X_0-X^\star\|_F .
\end{equation*}
Substituting this into~\eqref{eq:lm-smoothness} yields the momentum LM-loss
bound. The momentum bound is lower than the vanilla bound whenever
\begin{equation*}
    C_{\mathrm{mom}}^2\rho_{\mathrm{mom}}^{2N}
    <
    \rho_{\mathrm{vanilla}}^{2N},
\end{equation*}
equivalently
\begin{equation*}
    C_{\mathrm{mom}}
    <
    \left(
    \frac{\rho_{\mathrm{vanilla}}}
         {\rho_{\mathrm{mom}}}
    \right)^N .
\end{equation*}
Because $\rho_{\mathrm{vanilla}}/\rho_{\mathrm{mom}}>1$, the threshold
$N_0$ above is sufficient for this inequality to hold.
\end{proof}

\begin{remark}[Interpretation]
\label{rem:momentum-interpretation}
The assumptions are natural as a local model of residual-stream dynamics. First,
a Transformer used at inference follows a fixed trajectory of hidden states, so
linearizing each layer around the states it actually visits is the standard
first approximation to its local behavior. The resulting Jacobian has
token-feature modes with different effective rates, which is exactly the
situation where finite-depth residual updates can leave slow modes
under-corrected. Second, replacing this local operator by a self-adjoint
positive definite surrogate isolates the aligned component of the layer oracle:
the part that moves the representation toward a task-relevant state
$X^\star$. This is the favorable case for vanilla; if momentum improves even
there, the advantage is architectural rather than an artifact of a hostile
oracle. Third, the smooth readout assumption is natural because the final logits
are linear in the hidden state and cross-entropy is smooth on bounded-logit
neighborhoods. Near a local minimizer $X^\star$, smaller representation error
therefore gives a smaller local upper bound on language-modeling loss.

The theorem explains the empirical advantage of momentum-stream architectures,
including the prior YuriiFormer baseline and TMMFormer, over vanilla as a
forward-architecture effect. Momentum does not merely change how parameters
are trained; it changes the residual stream from a first-order map
$X_\ell\mapsto X_{\ell+1}$ to a second-order map
$(X_\ell,V_\ell)\mapsto(X_{\ell+1},V_{\ell+1})$. In the local linearized
regime, this gives a faster filter for slow token-feature modes. The
TMM-vs-YuriiFormer statement is an expressivity containment result: TMMFormer
can recover YuriiFormer by setting $\nu_\ell=1$, while learning $\nu_\ell$
gives a larger second-order filter class.
\end{remark}

%% file: appendix/landscape_measurement.tex
\section{Loss-Landscape Sharpness Measurement}
\label{sec:landscape-measurement}

This appendix gives the precise definitions of the sharpness diagnostics
summarized in the \emph{Loss Landscape} setup. For a variant with trained
parameters $\theta\in\mathbb{R}^{N}$ we write $\mathcal{L}(\theta)$ for the
mean token-level cross-entropy on a fixed set of $B$ held-out validation
minibatches, and $H=\nabla^{2}_{\theta}\mathcal{L}(\theta)$ for its Hessian.
The matrix $H$ is never instantiated; every quantity below uses only
Hessian--vector products (HVPs). All quantities are evaluated at the best
(lowest validation loss) checkpoint, on the corpus the model was trained on
(OpenWebText unless stated otherwise). The $B$ minibatches are sampled once
with a fixed random seed and reused for every variant and every probe, so all
models are compared on identical inputs.

\paragraph{Hessian--vector products.}
For a vector $v\in\mathbb{R}^{N}$ the HVP is computed by double backward
(the Pearlmutter trick):
\begin{equation}
    Hv \;=\; \nabla_{\theta}\!\big(\langle\,\nabla_{\theta}\mathcal{L}(\theta),\,v\,\rangle\big),
    \label{eq:hvp}
\end{equation}
i.e.\ a first backward pass produces $g=\nabla_{\theta}\mathcal{L}$ with the
computation graph retained, and a second backward pass through the scalar
$\langle g,v\rangle$ yields $Hv$. The loss is evaluated with the math
attention kernel, because the fused/flash attention kernels do not support the
required double backward. Each HVP is averaged over the $B$ fixed validation
minibatches.

\paragraph{Top Hessian eigenvalue.}
The dominant curvature $\lambda_{\max}(H)$ is estimated by power iteration on
the HVP operator~\citep{yao2020pyhessian}:
\begin{equation*}
    v_{0}\sim\mathcal{N}(0,I_{N}),\quad
    \lambda^{(k)}=v_{k}^{\top}Hv_{k},\quad
    v_{k+1}=\frac{Hv_{k}}{\lVert Hv_{k}\rVert},
    \label{eq:power-iter}
\end{equation*}
run for at most $T_{\mathrm{pow}}$ iterations and stopped early once
$|\lambda^{(k)}-\lambda^{(k-1)}|/|\lambda^{(k)}|<\tau$. A large
$\lambda_{\max}$ indicates a sharp direction in parameter space.

\paragraph{Hessian trace.}
The trace is estimated with Hutchinson's estimator using Rademacher probes
$v^{(j)}\in\{-1,+1\}^{N}$~\citep{yao2020pyhessian}:
\begin{equation*}
    \operatorname{tr}(H)\;\approx\;\frac{1}{P}\sum_{j=1}^{P}
        v^{(j)\top} H\, v^{(j)},
    \qquad v^{(j)}_{i}\sim\mathrm{Unif}\{-1,+1\},
    \label{eq:hutchinson}
\end{equation*}
which is unbiased because $\mathbb{E}\!\left[v v^{\top}\right]=I_{N}$ for
Rademacher $v$. We report the mean and standard deviation over $P$ probes.
The quantity compared across variants is the scale-normalized trace
$\operatorname{tr}(H)/N$, i.e.\ the mean curvature per parameter.

\paragraph{Filter-normalized loss curve.}
Following \citet{li2018visualizing}, we probe a random direction that is
normalized per parameter tensor, which removes the spurious scale invariance
of (pre-norm) networks. For each weight tensor $\theta^{(l)}$ we draw
$d^{(l)}\sim\mathcal{N}(0,I)$ and rescale
\begin{equation}
    d^{(l)}\;\leftarrow\;d^{(l)}\,
        \frac{\lVert\theta^{(l)}\rVert}{\lVert d^{(l)}\rVert},
    \label{eq:filter-norm}
\end{equation}
then evaluate the loss along
$\phi(\alpha)=\mathcal{L}\big(\theta+\alpha\,d\big)$ on an evenly spaced grid
of $G$ values $\alpha\in[-\alpha_{\max},\alpha_{\max}]$, restoring $\theta$
afterwards. We summarize flatness by the loss range
$\max_{\alpha}\phi(\alpha)-\min_{\alpha}\phi(\alpha)$; because $d$ is
filter-normalized, this range is comparable across models of different
scale.

\paragraph{Hyperparameters.}
Unless stated otherwise we use power iteration with $T_{\mathrm{pow}}=15$ and
tolerance $\tau=10^{-3}$, $P=10$ Hutchinson probes, and a curve grid of $G=11$
points over $\alpha\in[-0.5,0.5]$. The $B$ validation minibatches are drawn
with a fixed seed and shared across all variants. Lower $\lambda_{\max}$,
lower $\operatorname{tr}(H)/N$, and a smaller curve range each indicate a
flatter loss landscape.

%% file: appendix/forgetting_generalization.tex
\section{Forgetting and Generalization Measurement}
\label{sec:forgetting-generalization-measurement}

This appendix gives the precise protocol for the forgetting and generalization
diagnostics summarized in the \emph{Forgetting and Generalization} setup. All
losses are mean token-level cross-entropy, evaluated on a fixed set of
validation minibatches drawn with a shared random seed so that every variant
is scored on identical inputs.

\paragraph{Forgetting via sequential fine-tuning.}
Let $\mathcal{L}_{S}$ be the validation loss on the \emph{source} corpus (the
pretraining corpus). We fine-tune the pretrained checkpoint on a \emph{target}
corpus and measure the source-corpus loss before fine-tuning ($T_{0}$) and
after ($T_{1}$), and define
\begin{equation}
    \mathrm{forgetting} = \mathcal{L}_{S}(T_{1}) - \mathcal{L}_{S}(T_{0}),
    \label{eq:forgetting}
\end{equation}
the rise in the source-corpus loss caused by adapting to the target; lower is
better retention. We run both directions, OpenWebText$\to$TinyStories and
TinyStories$\to$OpenWebText.

To isolate the \emph{architecture} from the pretraining optimizer, every
variant is fine-tuned with the same fixed AdamW optimizer
($\mathrm{lr}=10^{-4}$, weight decay $0.01$, $\beta=(0.9,0.95)$) for $1000$
steps, regardless of the optimizer used during pretraining. The learning rate
follows a linear warmup over $W$ steps and then a cosine decay to a floor of
$0.1$ of the peak,
\begin{equation}
    \mathrm{lr}(t)=
    \begin{cases}
        \dfrac{t}{W}, & t<W,\\[1.1ex]
        0.1+0.9\cdot\tfrac{1}{2}\!\left(1+\cos\!\dfrac{\pi(t-W)}{T-W}\right),
            & t\ge W,
    \end{cases}
    \label{eq:ft-lr}
\end{equation}
with $T$ the total number of fine-tuning steps. Source and target losses are
evaluated every $100$ steps on $50$ fixed validation minibatches (seed $0$)
under \texttt{bfloat16} autocast.

\paragraph{Zero-shot cross-corpus generalization.}
Generalization is measured without any fine-tuning: the OpenWebText checkpoint
is evaluated directly on out-of-distribution corpora---WikiText-103
(validation), LAMBADA (OpenAI test split), and C4 (English validation)---as
well as on in-distribution OpenWebText validation as a reference. For a corpus
with mean cross-entropy $\overline{\mathrm{ce}}$ the perplexity is
\begin{equation}
    \mathrm{ppl}=\exp\big(\overline{\mathrm{ce}}\big),
    \label{eq:ppl}
\end{equation}
computed over a fixed set of minibatches (seed $0$, identical across variants).
The out-of-distribution score reported in the main text is the mean perplexity
over the three out-of-distribution corpora; lower is better.

\paragraph{Detailed results.}
Table~\ref{tab:forgetting-generalization} reports the per-direction forgetting
and per-corpus zero-shot perplexity behind
Figure~\ref{fig:properties-overview}(b--d), for the four variants discussed in
the main text. The broad ordering is consistent across columns: the momentum
variants (TMMFormer, YuriiFormer) forget the least and attain the lowest
perplexity on the out-of-distribution corpora, AdamFormer is intermediate,
and the vanilla transformer is worst, matching the flatness pattern of
Section~\ref{sec:landscape-measurement}. Within the two momentum variants,
YuriiFormer is slightly better on most robustness columns, while TMMFormer is
better on pretraining validation loss in the main results.

\begin{tablebox}{Forgetting and zero-shot generalization}
\setlength{\tabcolsep}{4pt}
\renewcommand{\arraystretch}{1.15}
\begin{tabular}{@{}lccccccc@{}}
\textbf{Variant} & \multicolumn{2}{c}{\textbf{Forgetting} $\downarrow$}
 & \multicolumn{5}{c}{\textbf{Perplexity} $\downarrow$} \\
\cmidrule(lr){2-3}\cmidrule(lr){4-8}
 & O$\to$T & T$\to$O & OWT & WT-103 & LMB & C4 & OOD \\
\midrule
Vanilla     & $0.83$ & $1.12$ & $20.11$ & $61.58$ & $43.19$ & $39.73$ & $48.17$ \\
AdamFormer  & $0.77$ & $0.99$ & $19.78$ & $59.31$ & $43.13$ & $39.56$ & $47.33$ \\
YuriiFormer & $0.67$ & $0.95$ & $18.81$ & $51.78$ & $41.55$ & $38.92$ & $44.08$ \\
TMMFormer   & $0.69$ & $0.93$ & $18.69$ & $53.77$ & $41.78$ & $38.70$ & $44.75$ \\
\end{tabular}
\end{tablebox}
\captionof{table}{Forgetting (source-corpus loss increase, both transfer
directions; O$=$OpenWebText, T$=$TinyStories) and zero-shot perplexity on
in-distribution OpenWebText and the out-of-distribution corpora
(WikiText-103, LAMBADA, C4) with their average (OOD). Lower is better
throughout.}
\label{tab:forgetting-generalization}

%% file: appendix/sam_wsd.tex
\section{Learning-Rate Schedule and Sharpness-Aware Minimization}
\label{sec:sam-wsd-measurement}

The interventions of the corresponding subsection change only TMMFormer's
parameter-training recipe; the architecture, data, and total step budget
($T=30{,}000$ on OpenWebText) are held fixed.

\paragraph{Warmup--stable--decay (WSD) schedule.}
The default recipe is warmup--cosine. The WSD alternative keeps a constant
peak learning rate through most of training and decays only at the end. With
warmup $W=3{,}000$ steps, decay start $D=25{,}000$, and floor
$\eta_{\min}=0.1$ of the peak, the learning-rate multiplier is
\begin{equation}
    \mathrm{lr}(t)=
    \begin{cases}
        t/W, & t<W,\\[0.3ex]
        1, & W\le t<D,\\[0.3ex]
        1-(1-\eta_{\min})\,\dfrac{t-D}{T-D}, & t\ge D.
    \end{cases}
    \label{eq:wsd}
\end{equation}

\paragraph{Sharpness-Aware Minimization (SAM).}
SAM~\citep{foret2020sharpness} replaces the gradient at $\theta$ with the
gradient at the worst-case point in a $\rho$-ball. Each step uses two
forward/backward passes on the \emph{same} minibatch: the first yields
$g=\nabla_{\theta}\mathcal{L}(\theta)$ and the ascent perturbation
\begin{equation}
    \epsilon^{\star}=\rho\,\frac{g}{\lVert g\rVert_{2}},\qquad \rho=0.05,
    \label{eq:sam}
\end{equation}
and the second evaluates $\nabla_{\theta}\mathcal{L}(\theta+\epsilon^{\star})$,
which is used for the parameter update before $\epsilon^{\star}$ is undone.
This is true SAM (the minibatch is shared across both passes), not the
$m$-sharpness variant, and roughly doubles the per-step cost.

\paragraph{SAWD.}
SAWD uses the WSD schedule of \eqref{eq:wsd} and turns SAM on only during the
decay phase ($t\ge D$). Since SAM then runs for the final ${\approx}1/6$ of
training, the overhead is ${\approx}17\%$ rather than ${\approx}2\times$. All
sharpness statistics for these runs use the diagnostic of
Appendix~\ref{sec:landscape-measurement}.